\documentclass[11pt]{article}
\usepackage[T1]{fontenc}
\usepackage{latexsym,amssymb,amsmath,amsfonts,amsthm}
\usepackage{graphics}
\usepackage{enumitem}
\usepackage{diagbox}
\usepackage{multirow}
\usepackage{makecell}
\usepackage{graphicx}
\usepackage{algorithmicx}
\usepackage{algorithm}
\usepackage{mathrsfs}
\usepackage{subfigure}
\usepackage{color}
\usepackage{hyperref}
\usepackage{cite}
\usepackage{accents}
\makeatletter
\def\widebar{\accentset{{\cc@style\underline{\mskip14mu}}}}

\makeatother
\hypersetup{hidelinks}
\newcommand{\bds}[1]{\boldsymbol{#1}}
\newtheorem{theorem}{Theorem}[section]

\newtheorem{remark}[theorem]{Remark}
\newtheorem{assumption}[theorem]{Assumption}

\definecolor{hw}{rgb}{1,0,0}
\definecolor{lk}{rgb}{0,0,1}
\begin{document}
\renewcommand{\theequation}{\arabic{section}.\arabic{equation}}

\title{\bf PE-CSNet: An equivariant network architecture with learnable patch-based sparse representation}
\date{}

\author{Kai Li\thanks{Department of Applied Mathematics, The Hong Kong Polytechnic University, Hung Hom, Hong Kong. (\texttt{kai1li@polyu.edu.hk}, \texttt{zhi.zhou@polyu.edu.hk}).}
\and Haitao Long\thanks{SKLMS and Academy of Mathematics and Systems Science, Chinese Academy of Sciences,
Beijing 100190, China and School of Mathematical Sciences, University of Chinese Academy of Sciences,
Beijing 100049, China (\texttt{longhaitao@amss.ac.cn}, \texttt{b.zhang@amt.ac.cn}).}
\and
Bo Zhang\footnotemark[2]
\and
Haiwen Zhang\thanks{SKLMS and Academy of Mathematics and Systems Science, Chinese Academy of Sciences,
Beijing 100190, China ({\tt zhanghaiwen@amss.ac.cn}).} 
\and 
Zhi Zhou\footnotemark[1]
}
\maketitle

\begin{abstract}
Compressive sensing (CS) enables accurate signal reconstruction from sparse measurements and is widely applied in medical imaging, remote sensing, and image compression.
However, designing an effective, task-specific sparse transform and the corresponding optimization procedure for high-quality CS remains challenging.
This process typically requires expert domain knowledge and laborious parameter tuning.
To address this issue, we present a Patch-based Equivariant deep unrolling architecture, termed PE-CSNet, for accurate CS recovery.
While traditional CS methods generally use predefined patch-based transform sparsity, we generalize this idea by incorporating learnable transform sparsity that adapts to the specific CS task through an optimization-driven process.
Specifically, we first establish a generalized patch-based CS model, which we solve via a block coordinate descent (BCD) algorithm.
The BCD solver is then unrolled into a deep neural network, where all parameters of both the CS model and solver are learned through end-to-end training.
To improve data efficiency, we introduce a stochastic equivariant training strategy that exploits the patch-wise structure of the network, enabling PE-CSNet to learn effectively even from limited data.
We further provide a simpler, parameter-shared version of PE-CSNet and briefly discuss its convergence as an iterative solver.
For practical applications, the network uses stage-specific (non-shared) parameters to enhance its expressive power and thereby improve its performance.
On the tasks of CS magnetic resonance imaging (CS-MRI) and CS coded diffraction patterns (CS-CDP), PE-CSNet achieves state-of-the-art accuracy with fast computational speed, outperforming traditional methods and existing deep unrolling methods.

{\bf Keywords:} compressive sensing, deep learning, algorithm unrolling, patch-based regularization, magnetic resonance imaging (MRI), convergence analysis
\end{abstract}

\section{Introduction}\label{S1}
\setcounter{equation}{0}

Compressive sensing (CS) challenges the traditional sampling paradigm by proving that sparse signals can be exactly recovered from far fewer samples than those required by the Shannon-Nyquist theorem \cite{CEJ08, DDL06}.
This recovery is achieved by exploiting inherent signal sparsity in a known transform domain.
The profound impact of this theory is evidenced by its successful applications in computational imaging, including the single-pixel camera \cite{EMP19}, ultrasound video segmentation \cite{LJZ16}, and most notably, magnetic resonance imaging (MRI) \cite{LMD08}, where it directly enables faster scans and reduces patient burden.
The core problem is mathematically formulated as follows.
Let $ \bds{x}\in\mathbb{C}^Q $ denote a signal with certain transform sparsity \cite{DDL06}.
Given an undersampling measurement matrix $ \Phi\in \mathbb{C}^{P\times Q}\ (P\ll Q) $, the observed data is $ \bds{y} := \Phi \bds{x} \in \mathbb{C}^P $.
A typical CS model recovers the unknown signal $ \bds{x} $ from $ \bds{y} $ by solving the following variational model \cite{DDL06}:
\begin{equation}\label{1}
\hat{\bds{x}} = \underset{\bds{x}}{\arg\min}\left\{\dfrac{1}{2}\|\Phi \bds{x} - \bds{y}\|_2^2 + \lambda g(\Psi \bds{x})\right\},
\end{equation}
where $ \Psi $ is a sparsifying transform (e.g., a wavelet transform), and $ g(\cdot) $ is a sparsity-promoting regularization function, such as an $ \ell_q\;(0\leq q\leq 1) $ norm or a mixed norm \cite{UMP11, NDT09}.
Here, $ \|\cdot\|_2 $ denotes the $ \ell_2 $ norm, which quantifies the data fidelity.
The first term in \eqref{1} enforces this fidelity, the second term promotes the sparsity of the transform coefficients, and the regularization parameter $\lambda>0$ balances these two terms.

Within the above framework, accomplishing a CS task requires the careful design of a sparsifying transform $\Psi$ and a regularization function $ g(\cdot) $ that encode the a priori knowledge of the unknown signals for the variational model \eqref{1}, followed by an efficient optimization algorithm to solve this model.
However, in practical applications, designing both components effectively is highly challenging, often demanding expert knowledge and laborious manual tuning \cite{ZJC23}.
In this work, we focus on the widely-used patch-based sparsity model.
To address the above challenges, we propose a Patch-based Equivariant deep architecture, PE-CSNet, which automatically learns both the patch-based transform sparsity and the optimization procedure directly from data.

Initially, sparse regularization in CS was applied directly in the image domain or in traditional transform domains, such as finite difference domain or wavelet domain \cite{DID04, HLC09, LCY09}.
These regularization strategies, however, usually provide insufficient sparse representations, failing to capture complex image structures \cite{QXG12, KFB11}.
This limitation has motivated the development of more sophisticated sparse models, such as non-local sparsity \cite{MJB09, QXH14, DWS14}, group or structured sparsity \cite{UMP11, QXG12, ZJZ14}, dictionary-based sparsity \cite{RSB10, ZZC16, CJL20} and low-rankness \cite{LSG11, HZN21, ZZW23}.
In particular, patch-based regularization has been extensively studied as a powerful framework that underlies many of these advances.
A key strategy in patch-based methods is non-local processing, which exploits self-similarity by grouping similar patches from different image regions for joint regularization \cite{DKF07, BAC10, QXH14, DWS14, ZJZ14, RYL22}.
For instance, the BM3D method \cite{DKF07} groups similar 2D patches into 3D arrays and applies collaborative filtering.
A general reconstruction model was proposed in \cite{QXH14}, which uses the patch-based non-local operator (PANO) to obtain sparse representations of patch groups.
In \cite{XYY16}, the authors used image patches to form a dictionary, and then recovered the image patches by computing their sparse coefficients.
The idea of non-local regularization has also been combined with low-rank modeling \cite{DWS14} and group sparse representation \cite{ZJZ14}, demonstrating promising results in diverse CS tasks.
Despite their success, these methods share a fundamental limitation: their effectiveness relies on predefined regularization models (e.g., non-local means, specific group structures, or fixed dictionaries), whose design demands considerable domain expertise, and inaccuracies in these handcrafted priors can limit their reconstruction performance \cite{CDD23}.

To solve the CS model \eqref{1}, a variety of iterative algorithms have been proposed, including the fast iterative shrinkage-thresholding algorithm (FISTA) \cite{BAT09}, the Chambolle-Pock algorithm \cite{CAP11}, the alternating direction method of multipliers (ADMM) \cite{SBN11} and block coordinate descent (BCD) \cite{WSJ15}.
The BCD algorithm, in particular, is a popular variable-splitting approach for solving the CS model \eqref{1} \cite{RSB15, WBR17}.
It works by introducing auxiliary variables and solving a sequence of simpler subproblems alternately.
While BCD is theoretically well-founded with convergence guarantees \cite{WSJ15, BAT13}, its practical application faces two major drawbacks.
First, it typically requires numerous iterations to converge, incurring high computational cost; second, its performance is sensitive to hyperparameters (e.g., step sizes and penalty parameters), which typically require empirical tuning.

The CS methods discussed above are model-driven approaches based on a variational formulation comprising a data-fidelity term and a sparsity-promoting regularizer.
Although these CS methods have a solid theoretical foundation \cite{DDL06, CEJ06}, their practical performance heavily depends on several manually selected components, including the sparsifying transform $\Psi$, the regularizer $g(\cdot)$, the regularization parameter $\lambda$, and algorithmic hyperparameters.
This reliance on manual tuning remains a significant challenge.

In recent years, deep learning has emerged as a powerful technique for CS reconstruction \cite{MAP15, ZKZW17, ZJG18, YYS20, YYW23, CZX24, CBZ25}.
Early efforts primarily employed classical network architectures, such as convolutional neural networks (CNNs), to directly learn the inversion of the acquisition process $ y = \Phi x $ \cite{MAP15, ZKZW17}.
Specifically, given a dataset $ \{(\bds{y}_i,\bds{x}_i)\}_{i=1}^S $ consisting of measurements $ \bds{y}_i $ and their corresponding true signal $ \bds{x}_i $, a deep network architecture $ f_\textrm{NN}(\cdot;\Theta) $ is trained to map $ \bds{y}_i $ to $ \bds{x}_i $ by minimizing the empirical loss:
\begin{equation}\label{2}
\widehat{\Theta} = \underset{\Theta}{\arg\min}\dfrac{1}{S}\sum_{i=1}^S L\left(\bds{x}_i, f_\textrm{NN}(\bds{y}_i;\Theta)\right),
\end{equation}
where $L$ is a loss function (e.g., $ \ell_2 $ loss) and $\Theta$ are trainable parameters in the deep neural network.
Such methods learn the reconstruction directly from data, thus avoiding the need to manually determine the transform sparsity and hyperparameters, and they are typically fast.
However, CS reconstruction relies on effectively capturing transform sparsity, which is inherently challenging for a conventional black-box network to learn.
Moreover, the success of these methods typically requires a large amount of training data, which is often prohibitive in applications such as accelerated MRI.
A natural strategy, therefore, is to combine model-driven methods with deep learning, bringing together the strengths of both approaches.
This philosophy directly aligns with that of algorithm unrolling \cite{MVL21}, which was initially introduced to obtain fast neural network approximations for sparse coding \cite{GKL10}.
A prominent example of applying unrolled networks to CS is presented in \cite{YYS20}, where the authors unrolled the ADMM algorithm into a deep network to optimize a generalized CS model with a learnable sparsifying transform.
All parameters of the CS model and the ADMM algorithm are learned via end-to-end training similar to \eqref{2}.
Notably, the gradient of the undetermined sparse regularizer is parameterized by a small sub-network comprising convolutional and nonlinear activation layers.
Similarly, the iterative shrinkage-thresholding algorithm (ISTA) was unrolled into a deep architecture dubbed ISTA-Net for CS \cite{ZJG18}.
In this network, the sparsifying transform $ \Psi $ in \eqref{1} is implemented as a combination of two linear convolutional operators separated by a rectified linear unit (ReLU), while the regularizer $ g(\cdot) $ in \eqref{1} is the $ \ell_1 $ norm.
In \cite{YYW23}, the authors proposed a novel non-convex reconstruction model for blind parallel MRI, where the magnetic resonance image and multichannel sensitivity maps are jointly estimated.
They regularized both the image and sensitivity maps, and introduced a proximal alternating linearized minimization (PALM) algorithm to solve the resulting optimization problem.
The authors in \cite{CZX24} proposed a modified training approach and introduced termination criteria for deep unrolling architectures, thereby establishing unrolling as an iterative regularization technique.
Based on this framework, they unrolled the proximal gradient descent algorithm for compressive sensing MRI.
Recently, PCNet \cite{CBZ25} was proposed, which features a collaborative sampling operator and a deep unrolled reconstruction network enhanced by window-based Transformer modules.
The method leverages a large-scale dataset and improved training strategies for effective optimization.
For a broader overview of deep learning-based CS methods, we refer the reader to the surveys \cite{BMB18, ASM19, MVL21, ZJC23, ZZW23, ZLZ26} and the references therein.

The aforementioned deep unrolling methods primarily aim to learn a general sparsifying transform within a relatively constrained network hypothesis space.
As a result, such approaches may not yield sufficiently powerful sparse representations and potentially require extensive training data.
Building on this insight, we aim to learn an adaptive patch-based sparsifying transform directly from data, thereby seeking to combine the structural advantage of patch-based CS models with the benefit introduced by deep learning methods.

Recently, equivariant imaging (EI) has emerged as a powerful framework for imaging inverse problems, enabling more data-efficient learning by preserving the inherent symmetries of the data throughout the reconstruction \cite{CDD23}.
In typical equivariant imaging methods, the signal space $\mathcal{X}\subset\mathbb{C}^N$ is assumed to be invariant under a group of transformations $ \{\mathcal{T}_i\}_{i=1}^T $ (e.g., shifts, rotations and flips).
That is, for any true signal $ \bds{x}\in \mathcal{X} $, all transformed signals $ \mathcal{T}_i \bds{x}\; (i=1,2,\ldots,T) $ also lie in $\mathcal{X}$.
Given a nonlinear reconstruction map that is trained as in \eqref{2}, to ensure that the reconstruction is compatible with these symmetries, it is desired that the composition $ f_\textrm{NN}\circ\Phi $ be equivariant to the transformations, that is,
\begin{equation}\label{3}
f_\textrm{NN}(\Phi(\mathcal{T}_i\bds{x}); \Theta)=\mathcal{T}_i f_\textrm{NN}(\Phi(\bds{x}), \Theta)\approx \mathcal{T}_i\bds{x},\quad i=1,2,\ldots,T.
\end{equation}
This principle has enabled fully unsupervised frameworks for various CS tasks \cite{CDT21, CDT22}, achieving satisfactory reconstruction quality.
The Flipping and Rotation Invariant Sparsifying Transform (FRIST) was introduced in \cite{WBR17} to learn a sparsifying transform that is invariant to flips and rotations, thereby better representing natural images containing textures with diverse geometric orientations.
In \cite{CTS16}, group equivariant CNNs (GCNNs) were designed with built-in invariances, such as rotations and flips, for general computer vision tasks.
A closely related concept is data augmentation (DA) \cite{LKV15}, widely adopted across machine learning for its simplicity and effectiveness.
For instance, in the case of AlexNet \cite{KAS17}, a pioneering deep learning model for image classification, the ImageNet dataset was expanded by a factor of 2,048 using DA, resulting in approximate equivariance \cite{LKV15}.
This involves generating new training samples by applying transformations such as flips, scaling, and rotations, while preserving their original labels.
The authors in \cite{FZH21} introduced a DA pipeline for accelerated MRI reconstruction and evaluated its effectiveness in lowering the data requirement under diverse conditions.
In order to improve deep unrolling methods, the authors in \cite{FJX24} proposed a rotation equivariant proximal network that effectively embeds rotation symmetry priors into the deep unrolling framework.
Such an effective equivariant architecture is applied to various image restoration tasks, such as image de-raining and computed tomography (CT).

Motivated by the complementary strengths of model-based sparse regularization, deep unrolled architectures, and EI techniques, we introduce a learnable, patch-based unrolled network that adaptively learns sparsifying transforms from data within a CS optimization framework.
To enhance its performance with limited data, we integrate EI techniques into the architecture.
This combination enables the network to cope with small datasets and achieve better generalization.
To be more specific, similar to \eqref{1}, our framework begins with a generalized patch-based CS model that focuses on the image reconstruction setting:
\begin{equation}\label{4}
\hat{\bds{x}} = \underset{\bds{x}}{\arg\min}\left\{\dfrac{1}{2}\|\Phi \bds{x}-\bds{y}\|_2^2 + \sum_{i,j}\mathcal{R}(D_{ij}\bds{x})\right\},
\end{equation}
where the image (with a slight abuse of notation) $ \bds{x}\in \mathbb{C}^{M\times M} $ can be viewed as a $ (N_p\times N_p) $ grid of non-overlapping patches, each of size $ (m\times m) $, yielding $ N_p = M/m $.
The operator $ D_{ij}\;(1\leq i,j\leq  N_p,\;i,j\in \mathbb{N}^+) $ denotes a linear operator that extracts the $ (i,j) $-th non-overlapping patch from the image $ \bds{x} $ (see Figure \ref{F1} for illustration), and $\mathcal{R}$ is a learnable sparse regularization function.
The full image can be reconstructed linearly from the complete set of patches $ \{D_{ij}\bds{x}\} $, an operation we denote by $ \mathtt{Reassemble}(\cdot) $, i.e., $ \mathtt{Reassemble}(\{D_{ij}\bds{x}\}) = \bds{x} $.
Notably, the proposed formulation \eqref{4} generalizes many traditional patch-based transform sparsity, such as non-local sparsity \cite{QXH14}, group sparsity \cite{MJB09} and structured sparsity \cite{QXG12}.
To solve the generalized patch-based CS model \eqref{4}, we apply a BCD iterative algorithm.
By unrolling and generalizing the BCD iterations, we obtain a deep network architecture termed PE-CSNet, where all undetermined parameters of the generalized CS model and the BCD solver become the network's trainable weights, denoted by $\Theta$.
The resulting network architecture defines a parameterized mapping $ f_\textrm{PE}(\cdot; \Theta) $, illustrated in Figure \ref{F2} (see Section \ref{S3} for detailed description).
Guided by the principle of EI, the reconstruction mapping $ f_\textrm{PE}(\cdot; \Theta) $ is learned by minimizing the following loss function:
\begin{equation}\label{5}
\mathcal{L}_\textrm{PE}(\Theta) := \mathcal{L}_\textrm{DC}(\Theta) + \beta\mathcal{L}_\textrm{EQ}(\Theta)
\end{equation}
with the data consistency term (DC) and equivariance term (EQ) defined respectively as
\begin{equation}\label{6}
\mathcal{L}_\textrm{DC}(\Theta) := \dfrac{1}{S}\sum_{i=1}^SL(\bds{x}_i, f_\textrm{PE}(\bds{y}_i;\Theta)),
\end{equation}
\begin{equation}\label{7}
\mathcal{L}_\textrm{EQ}(\Theta) := \dfrac{1}{ST}\sum_{i=1}^S\sum_{j=1}^TL(\mathcal{T}_j\bds{x}_i, f_\textrm{PE}(\Phi(\mathcal{T}_j\bds{x}_i); \Theta)).
\end{equation}
Here, $ \mathcal{L}_\textrm{DC} $ enforces data consistency, $ \mathcal{L}_\textrm{EQ} $ enforces system equivariance based on \eqref{3}, and $\beta$ is a trade-off parameter.
In the testing phase, a sparsely sampled measurement is processed by $ f_\textrm{PE} $ to produce a high-quality reconstruction with fast inference speed.
The main contributions of this paper are summarized as follows:
\begin{itemize}
\item A novel deep unrolling network, PE-CSNet, that adaptively learns patch-based sparse representations, a powerful prior widely used in model-driven methods but rarely explored in deep unrolling.
Unlike typical unrolling methods that learn generic transforms, our network explicitly leverages an intrinsic patch-based structure and is derived from a generalized patch-based model.
This formulation allows the patch-based sparsifying transform to be directly and optimally learned from data, rather than being predefined.

\item Equivariant training strategy that exploits the intrinsic patch structure for data efficiency.
We propose an equivariant imaging (EI) training strategy for PE-CSNet.
By leveraging the network's intrinsic patch-based architecture, this strategy enforces consistency across geometric transformations at the patch level.
This leads to highly data-efficient learning, significantly reducing the need for large training datasets.

\item Extensive experiments validate that PE-CSNet achieves promising performance across both compressive sensing magnetic resonance imaging (CS-MRI) and compressive sensing coded diffraction pattern (CS-CDP), outperforming state-of-the-art traditional and deep learning methods.
It achieves robust generalization, even with limited training data.
\end{itemize}

The rest of the paper is organized as follows.
In Section \ref{S2}, we introduce the generalized patch-based CS model and derive its BCD solver.
Section \ref{S3} presents the PE-CSNet architecture, details its stochastic equivariant training strategy along with the network inference procedure, and discusses the convergence of a simplified, parameter-shared version of PE-CSNet.
Note that our experiments use the more expressive stage-specific (non-shared) parameters for PE-CSNet.
Section \ref{S4} evaluates the performance of PE-CSNet on CS-MRI and CS-CDP tasks through extensive experiments.
Finally, some conclusions and remarks are given in Section \ref{S5}.

\begin{figure}[!htbp]
\centering
\includegraphics[width=.66\linewidth]{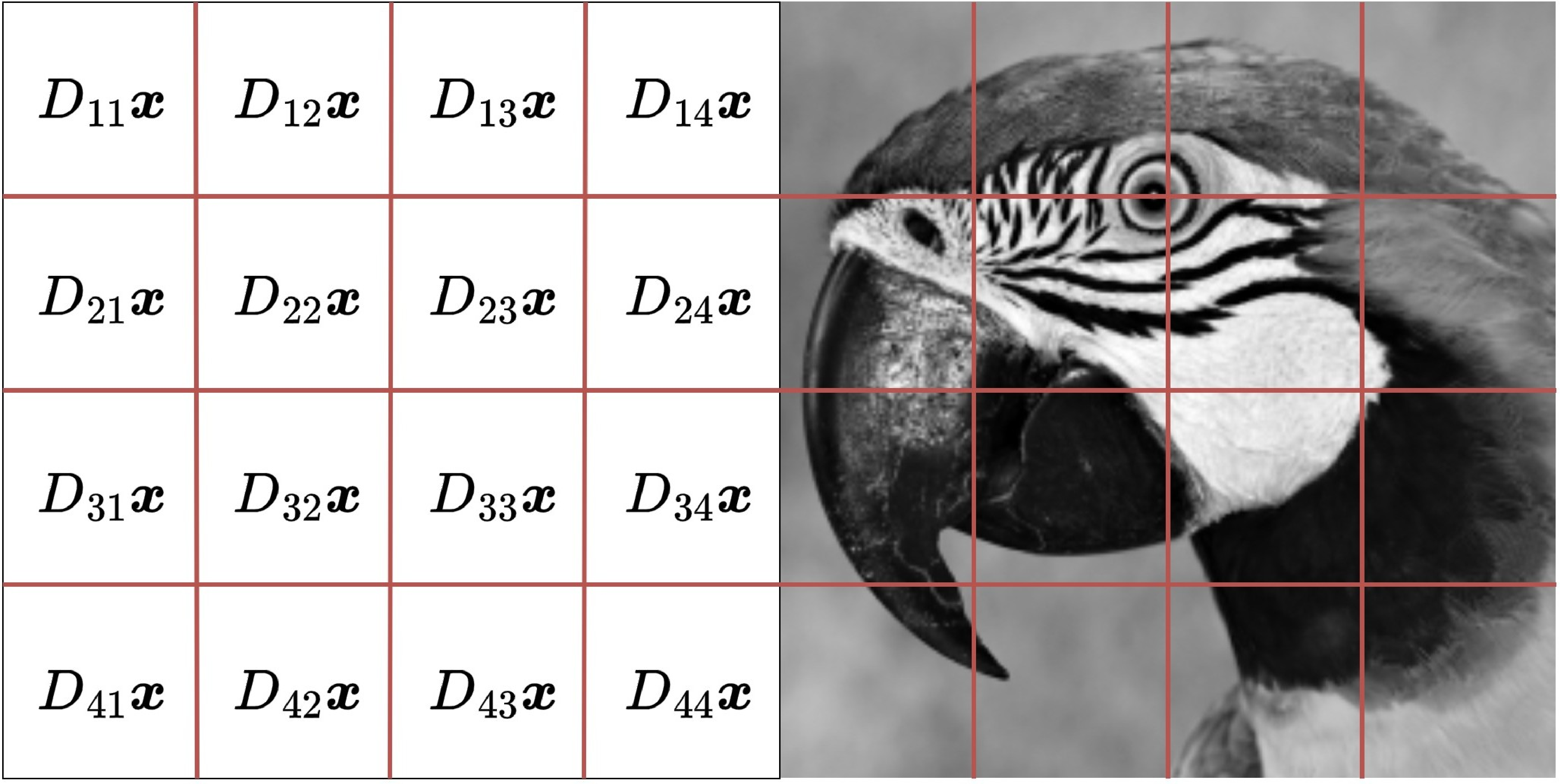}
\caption{Visualization of the patch-extraction operator $ D_{ij} $ for a $ (4\times 4) $ patch grid ($ N_p=4 $).
}\label{F1}
\end{figure}

\section{Generalized patch-based compressive sensing model}\label{S2}
\setcounter{equation}{0}
In this section, we first introduce a patch-based CS model that generalizes many traditional patch-based models.
By introducing auxiliary variables via a penalty function, we apply the BCD algorithm to solve this model.
In Section \ref{S3}, we further derive the corresponding network architecture of PE-CSNet by unrolling and generalizing the BCD algorithm.

\subsection{Model}
While patch-based regularization is widely used in model-driven CS, many such methods can be summarized in the following general form \cite{MJB09, QXG12, QXH14, ZZC16, RYL22}:
\begin{equation}\label{8}
\hat{\bds{x}} = \underset{\bds{x}}{\arg\min}\left\{\dfrac{1}{2}\|\Phi \bds{x}-\bds{y}\|_2^2 + \lambda\sum_{i,j}\|\Psi(D_{ij}\bds{x})\|_1\right\},
\end{equation}
where $ \Psi $ is a (typically predefined) sparsifying transform of the $ (i,j) $-th patch of $ x $ and $\lambda>0$ is the regularization parameter.
The transform $ \Psi $ in \eqref{8} includes common choices such as wavelet transforms \cite{QXG12, QXH14}, learned dictionaries \cite{ZZC16}.
Note that the patch-based CS model \eqref{8} is a special case of \eqref{1} that focuses on patch-based transform sparsity.
To better exploit the a priori information within image patches, we generalize the model in \eqref{8} by replacing the sparsifying transform $ \Psi $ and the $\ell_1$ penalty with a more general regularizer.
The generalized patch-based CS model can be formulated as follows:
\begin{equation}\label{9}
\hat{\bds{x}} = \underset{\bds{x}}{\arg\min}\left\{\dfrac{1}{2}\|\Phi \bds{x}-\bds{y}\|_2^2 + \sum_{i,j}\mathcal{R}(D_{ij}\bds{x})\right\},
\end{equation}
where $\mathcal{R}$ is a (possibly) nonlinear regularization function that enforces sparsity on the extracted patches.
We note that the regularization parameter $\lambda$ is absorbed into the definition of $\mathcal{R}$.
Notably, the previous model \eqref{8} is a special case of our generalized model \eqref{9}, with $\mathcal{R}(\cdot) := \lambda\|\Psi(\cdot)\|_1$.
In contrast to traditional model-driven methods where $\mathcal{R}$ is predefined, our approach determines $\mathcal{R}$ via an end-to-end learning strategy.

\subsection{BCD solver}\label{S22}
The BCD algorithm provides an efficient solver for the generalized patch-based model \eqref{9}, owing to its separable structure across patches.
To apply the BCD algorithm, we first employ the variable splitting and quadratic penalty technique \cite{QXH14, YYS20} to the CS model \eqref{9}.
By introducing an auxiliary variable $ \bds{z} $ in the image domain, we convert \eqref{9} into the constrained form:
\begin{equation}\label{10} \underset{\bds{x},\bds{z}}{\min}\left\{\dfrac{1}{2}\left\|\Phi \bds{x}-\bds{y}\right\|_2^2 + \sum_{i,j}\mathcal{R}(D_{ij}\bds{z})\right\},\quad \textit{s.t.}\; \bds{z} = \bds{x}.
\end{equation}
Applying the quadratic penalty technique, we obtain a relaxed unconstrained version of \eqref{10} as follows:
\begin{equation}\label{11} \mathscr{L}(\bds{x},\bds{z}; \rho) := \dfrac{1}{2}\|\Phi \bds{x}-\bds{y}\|_2^2 + \sum_{i,j}\mathcal{R}(D_{ij}\bds{z}) + \dfrac{\rho}{2}\|\bds{x}-\bds{z}\|_2^2,
\end{equation}
where $ \rho>0 $ is the penalty factor.
To minimize \eqref{11}, we use the BCD, which alternately updates $ \bds{x} $ and $ \bds{z} $ by solving the following subproblems:
\begin{equation}\label{12}
\left\{ \begin{array}{ll}
\underset{\bds{x}}{\arg\min}\;\dfrac{1}{2}\|\Phi \bds{x} - \bds{y}\|^2_2 +  \dfrac{\rho}{2}\|\bds{x} - \bds{z}\|^2_2,\\
\underset{\bds{z}}{\arg\min}\;\sum_{i,j}\mathcal{R}(D_{ij}\bds{z}) + \dfrac{\rho}{2}\|\bds{x}-\bds{z}\|_2^2.
\end{array} \right.
\end{equation}
Since $ \|\bds{x}-\bds{z}\|_2^2 = \sum_{i,j}\|D_{ij}\bds{x} - D_{ij}\bds{z}\|_2^2 $, the $ \bds{z} $-subproblem above can be decoupled as follows:
\begin{equation*}
\underset{D_{ij}\bds{z}}{\arg\min}\;\mathcal{R}(D_{ij}\bds{z}) + \dfrac{\rho}{2}\|D_{ij}\bds{x}-D_{ij}\bds{z}\|_2^2,\quad \forall i,j.
\end{equation*}
Let $ I $ denote the identity operator, with the dimension understood from context.
Then, the subproblems of \eqref{12} admit the following closed-form updates:
\begin{equation}\label{13}
\left\{ \begin{array}{ll}
\bds{X}^{(n)}: \bds{x}^{(n)} = (\Phi^H\Phi + \rho I)^{-1}[\Phi^H\bds{y} + \rho \bds{z}^{(n-1)}],\\
\cdots\cdots\\
\bds{Z}^{(n)}_{ij}:D_{ij}\bds{z}^{(n)} \in \mathrm{prox}_{(1/\rho)\mathcal{R}}(D_{ij}\bds{x}^{(n)}),\\
\cdots\cdots
\end{array} \right.
\end{equation}
where $ \mathrm{prox}_{(1/\rho)\mathcal{R}}(\cdot) $ is the proximal operator of $ \mathcal{R} $ with parameter $ 1/\rho $, and the superscript $ H $ denotes Hermitian transpose.
Here, $ \bds{z}^{(n)} $ can be obtained by reassembling its patches $ D_{ij}\bds{z}^{(n)} $, i.e., $ \bds{z}^{(n)} = \mathtt{Reassemble}(\{D_{ij}\bds{z}^{(n)}\}) $ (see Section \ref{S1}).
Note that the regularization function $\mathcal{R}$ in \eqref{9} encodes the patch-based sparse a priori information, which is unknown and difficult to choose in advance.
Consequently, its proximal operator $ \mathrm{prox}_{(1/\rho)\mathcal{R}} $ lacks an explicit, implementable form.
To address this, we approximate the proximal operator via a deep neural network, which is achieved by unrolling and generalizing the BCD iterations derived above into a learnable architecture, as detailed in the next section.

\begin{remark}\label{R4}
Suppose that for each $ \rho^{(k)} $, the pair $ (\bds{x}_k, \bds{z}_k) $ is an exact minimizer of $ \mathscr{L}(\bds{x},\bds{z}; \rho^{(k)}) $ defined in \eqref{11}, and that $ \rho^{(k)}\to\infty $.
It is known from \cite[Chapter 17.1]{NJW06} that every limit point of $ \{(\bds{x}_k, \bds{z}_k)\} $ is a solution to \eqref{10}.
Therefore, solving \eqref{11} with sufficiently large $\rho$ provides a reasonable approximation to the exact solution of \eqref{10}.
This justifies our use of \eqref{11}.
\end{remark}

\begin{remark}\label{R1}
The update for $ \bds{x} $ in \eqref{13} requires solving a linear system involving $ (\Phi^H\Phi + \rho I) $, which can be computationally expensive for general measurement matrices.
In this work, we focus on structurally random matrices (SRMs) \cite{DTT12}---such as partial Fourier matrices or coded diffraction patterns (CDPs) \cite{LHL25}---which are widely used in CS.
For these SRMs, the linear system in \eqref{13} can be solved efficiently using a closed-form solution \cite{YYS20}, which we use in our implementation.
\end{remark}

\section{PE-CSNet for CS imaging}\label{S3}
\setcounter{equation}{0}

The practical deployment of the BCD solver \eqref{13} faces three major challenges: (i) despite having theoretical guarantees, it often requires numerous iterations to converge; (ii) the design of an effective sparsity regularizer $\mathcal{R}$ and its proximal operator is highly challenging; and (iii) the penalty parameter $\rho$ lacks a principled setting and requires extensive tuning.
To address these issues, we unroll the BCD iterations into a deep architecture, termed PE-CSNet, where the proximal operator and algorithm parameters are learned automatically from data, replacing handcrafted designs.

A common implementation practice is to specify a maximum number of iterations $ N $ in advance, which controls computational effort while allowing the iterative process to converge towards a satisfactory reconstruction.
Starting from $ \bds{x}^{(0)} := \Phi^H\bds{y} $, the BCD iteration \eqref{13} with $ N $ steps naturally defines a nonlinear mapping from undersampled measurements $ \bds{y} $ to the reconstruction $ \bds{x}^{(N)} $.
This motivates us to unroll this mapping into a deep architecture, PE-CSNet, by replacing the mathematical operations in each BCD iteration with learnable neural modules.
To be more specific, the updates $ \bds{X}^{(n)} $ and $ \bds{Z}^{(n)}_{ij} $ in \eqref{13} are implemented by modules $ (\bds{X}^{(n)}) $ and $ (\bds{Z}^{(n)}_{ij}) $, respectively.
As shown in Figure \ref{F2} (top), the resulting network's data flow mirrors the BCD iterative process, with each network stage corresponding to a single BCD step.
See Algorithm \ref{PE-CSNet} for the pseudo code of our PE-CSNet.
The architecture of these two modules is detailed as follows.

\begin{figure}[!htbp]
\centering
\includegraphics[width=\linewidth]{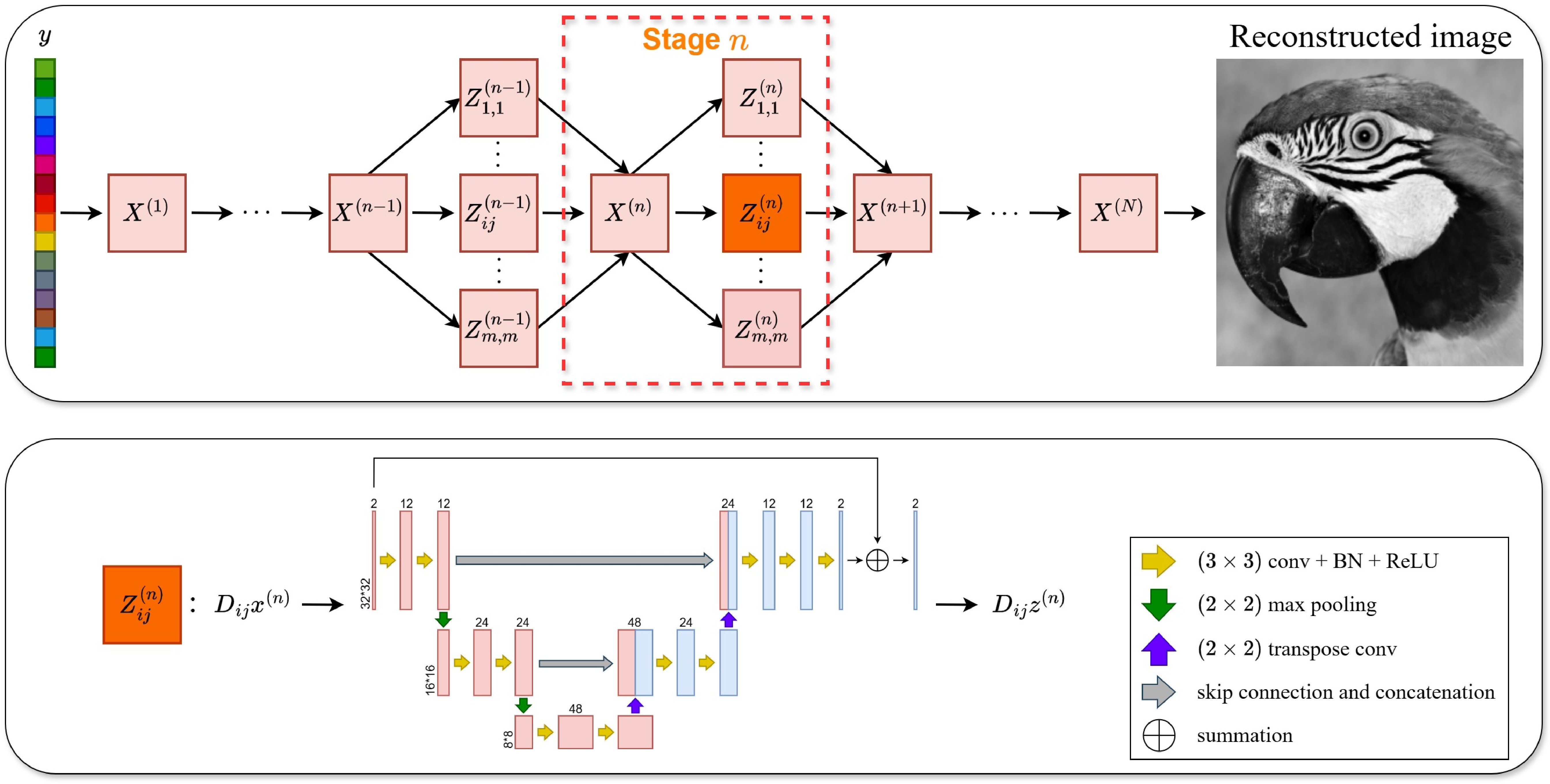}
\caption{The architecture of PE-CSNet.
The top part shows the overall unrolled architecture, which alternates between the reconstruction module $ (\bds{X}^{(n)}) $ and the auxiliary patch update module $ (\bds{Z}^{(n)}_{ij}) $.
The bottom part provides a detailed view of the residual U-Net that implements $ (\bds{Z}^{(n)}_{ij}) $.
}\label{F2}
\end{figure}

\textit{Reconstruction module $ (\bds{X}^{(n)}) $}.
This module implements the $ \bds{x} $-update in \eqref{13}.
Given the auxiliary variable $ \bds{z}^{(n-1)} $ from the previous stage, it computes the updated reconstruction $ \bds{x}^{(n)} $.
The output of $ (\bds{X}^{(n)}) $ is given by:
\begin{equation}\label{14}
\bds{x}^{(n)} := \left(\Phi^H\Phi + \rho^{(n)}I\right)^{-1}\left(\Phi^H \bds{y} + \rho^{(n)}\bds{z}^{(n-1)}\right),
\end{equation}
where $ \rho^{(n)}>0 $ is a learnable parameter at the $ n $-th stage.
While traditional penalty methods typically use a predefined increasing sequence $ \{\rho^{(k)}\} $ (see Remark \ref{R4}), in the proposed PE-CSNet, $ \rho^{(n)} $ becomes a learnable parameter that adapts to each stage through end-to-end training.
The resulting $ \bds{x}^{(n)} $ is then passed to the subsequent auxiliary patch update modules $(\bds{Z}_{ij}^{(n)})$ within the same stage for further processing.
Recall that the initial stage ($ n=0 $) is initialized with $ \bds{x}^{(0)} := \Phi^H\bds{y} $.
For structurally random matrices (e.g., partial Fourier matrices or CDPs), the matrix inversion in \eqref{14} can be computed efficiently using a closed-form solution, as discussed in Remark \ref{R1}.

\textit{Auxiliary patch update module $ (\bds{Z}^{(n)}_{ij}) $}.
This module implements the update of $ D_{ij}\bds{z} $ in \eqref{13}, which corresponds to the $ (i,j) $-th patch of the auxiliary variable $ \bds{z}^{(n)} $.
Since this module enforces the patch-based transform sparsity, which is crucial for high-quality CS recovery, we implement it using a residual U-Net, a powerful architecture for learning structured image representations.
Our variant (Figure \ref{F2}, bottom) builds upon the classical U-Net \cite{ROF15} and is similar to the design in \cite{LZZ24}, but is tailored for patch-based sparse modeling.
This module takes the patch $ D_{ij}\bds{x}^{(n)} $ as input and produces the updated patch representation $ D_{ij}\bds{z}^{(n)} $.
The output of $ (\bds{Z}^{(n)}_{ij}) $ is given by:
\begin{equation}\label{15}
D_{ij}\bds{z}^{(n)} := \mathcal{U}_{\alpha,\bds{\theta}}^{(n)}\left(\mathcal{P}_r\left(D_{ij}\bds{x}^{(n)}\right)\right),
\end{equation}
where $ \mathcal{P}_r $ is a fixed projection operator (with a predefined radius $ r $, see Subsection \ref{S421}) that ensures the input to the subsequent network $ \mathcal{U}_{\alpha,\bds{\theta}}^{(n)} $ remains bounded:
\begin{equation}\label{19}
\mathcal{P}_r(\bds{x}) := \left\{ \begin{array}{ll}
r\dfrac{\bds{x}}{\|\bds{x}\|_2},\quad &\|\bds{x}\|_2> r,\\
\bds{x},\quad &\|\bds{x}\|_2\leq r.
\end{array} \right.
\end{equation}
Here, $ \mathcal{U}_{\alpha,\bds{\theta}}^{(n)} $ is a learnable residual network that approximates the proximal operator $ \mathrm{prox}_{f^{(n)}} $ with $ f^{(n)}:= \left(1/\rho^{(n)}\right)\mathcal{R}^{(n)} $.
It serves as the learnable module and is constructed as
\begin{equation}\label{20}
\mathcal{U}_{\alpha,\bds{\theta}}^{(n)} := I + \alpha^{(n)}\mathcal{U}(\cdot;\bds{\theta}^{(n)}),
\end{equation}
in which $ \mathcal{U}(\cdot;\bds{\theta}^{(n)}) $ is a convolutional neural network (a U-Net, see Figure \ref{F2}, bottom), and
the parameters $ \alpha^{(n)} $ and $ \bds{\theta}^{(n)} $ are learned end-to-end.
As described in Section \ref{S22}, the full auxiliary variable is then reconstructed by $\bds{z}^{(n)} = \mathtt{Reassemble}(\{D_{ij}\bds{z}^{(n)}\}) $, where the updated patches $ \{D_{ij}\bds{z}^{(n)}\} $ are computed by \eqref{15} (see Section \ref{S1} for the definition of $ \mathtt{Reassemble}(\cdot) $).
The reconstructed $ \bds{z}^{(n)} $ subsequently serves as the input to the reconstruction module $ (\bds{X}^{(n+1)}) $ in the next stage, thereby completing one full BCD iteration within the unrolled network.
The detailed structure of the residual U-Net $ \mathcal{U}_{\alpha,\bds{\theta}}^{(n)} $, which is the core of $ (\bds{Z}^{(n)}_{ij}) $, is shown at the bottom of Figure \ref{F2}.
Each red and blue item represents a multichannel feature map; the number of channels is shown at the top of the volume, and the length and width are indicated at the lower-left.
Arrows denote different operations, as explained in the legend.
For more details of U-Net, see \cite{ROF15, LZZ24}.

In summary, PE-CSNet is constructed by unrolling the BCD iterations \eqref{13} of the generalized patch-based model \eqref{9}.
The overall network, denoted as $ f_\mathrm{PE}(\cdot;\Theta) $, is parameterized by $\Theta$, which collects all learnable parameters across its $ N $ stages.
These parameters consist of the penalty coefficients $ \{\rho^{(n)}\}_{n=1}^N $ in the reconstruction modules $ \{(\bds{X}^{(n)})\}_{n=1}^{N} $ and the weights $ \{\alpha^{(n)}, \bds{\theta}^{(n)}\}_{n=0}^{N-1} $ of the residual U-Nets within the patch update modules $ \{(\bds{Z}^{(n)})\}_{n=0}^{N-1} $.
As noted earlier, the reconstruction module $ (\bds{X}^{(0)}) $ is an initialization defined by $ \bds{x}^{(0)} := \Phi^H\bds{y} $.
Assuming each network $ \mathcal{U}(\cdot;\bds{\theta}^{(n)}) $ has $ N_{\bds{\theta}} $ parameters, the total number of trainable parameters of PE-CSNet $ f_\mathrm{PE}(\cdot;\Theta) $ is $ N\times (2+N_{\bds{\theta}}) $.
It should be noted that PE-CSNet retains the algorithmic structure of the traditional BCD solver but replaces the handcrafted sparse a priori information with a learnable residual U-Net.
This allows the sparse representation to be learned adaptively from data rather than predefined.

\subsection{Stochastic equivariant training strategy and network inference}\label{S31}
The parameters $ \Theta $ of PE-CSNet $ f_\mathrm{PE}(\cdot;\Theta) $ are trained end-to-end to reconstruct the ground truth image $ \bds{x} $ from its undersampled measurement $ \bds{y} $.
Given a dataset $ \mathcal{S} := \{(\bds{y}_i,\bds{x}_i)\}_{i=1}^S $, we adopt an equivariant learning strategy to improve data efficiency.
This strategy exploits the invariance of the ground truth image space $ \mathcal{X} $ under a predefined set of transformations.
We construct this transformation set as the composition of two fundamental types: discrete rotations/flips and limited translations.
First, we define a transformation group $ G := \{g_i\}_{i=1}^{|G|} $ generated by $ 90^\circ $ rotations and horizontal flips, which has $ |G|=8 $ elements.
Second, we define a set of translations $ L := \{l_i\}_{i=1}^{|L|} $, where each $ l_i $ represents a 2D translation by integer distances $ d^{(i)}_x $ and $ d^{(i)}_y $ in the horizontal and vertical directions, respectively.
To ensure the translation remains within a local patch of side length $ m $ (i.e., the same patch size as defined in Section \ref{S1}), $ d^{(i)}_x $ and $ d^{(i)}_y $ are restricted to:
\begin{equation*}
d^{(i)}_x, d^{(i)}_y\in \left[-m+\lfloor \dfrac{m}{2}\rfloor, \lfloor \dfrac{m}{2}\rfloor-1\right]\cap \mathbb{Z},
\end{equation*}
where $ \lfloor \cdot\rfloor $ denotes the floor function.
This yields $ |L| = m^2 $ distinct translations.
Any blank regions produced by translation are padded via reflection.
The full transformation set is then obtained by composing every translation with every transformation from $ G $, i.e., $ \{l_i\circ g_j: l_i\in L,\; g_j\in G\} $.
We denote this composed set simply as $ \{\mathcal{T}_i\}_{i=1}^T $, where $ T = |G|\times |L| = 8m^2 $.
As discussed in Section \ref{S1}, the equivariant training strategy uses the loss function defined in \eqref{5}:
\begin{equation}\label{16}
\mathcal{L}_\textrm{PE}(\Theta) = \mathcal{L}_\textrm{DC}(\Theta) +\beta\mathcal{L}_\textrm{EQ}(\Theta),
\end{equation}
where $ \mathcal{L}_\textrm{DC} $ enforces data consistency (DC) and $ \mathcal{L}_\textrm{EQ} $ serves as the equivariant regularization term for $ f_\textrm{PE}(\cdot;\Theta) $, as given in \eqref{6} and \eqref{7}, respectively.
Here, $ L(\cdot) := ||\cdot||^2_2 $ in both terms, and $ \beta>0 $ is a trade-off parameter.
Given that each ground truth image $ \bds{x}_i\; (i = 1,\ldots,S) $ is of size $ (M\times M) $ (i.e., the image dimension defined in Section~\ref{S1}), the proposed training strategy significantly increases the number of effective samples for learning the auxiliary patch update module $ \mathcal{U}_{\alpha,\bds{\theta}}^{(n)} $.
Specifically, each image is divided into $ (M/m)^2 $ patches, and each patch is further transformed by the $ T = 8m^2 $ operations in the set $ \{\mathcal{T}_i\}_{i=1}^T $.
Consequently, the effective number of training samples for this module is scaled by a factor of $ T\times (M/m)^2 = 8\times M^2 $.
In our numerical experiments (see Section \ref{S4}), we set $ M=256 $, which yields a scaling factor $ 8\times 256^2 = 524,288 $.
Although this introduces substantial redundancy within a single image, training PE-CSNet with $ \mathcal{L}_\textrm{PE} $ in \eqref{16} enables the network to exploit patch-based a priori information comprehensively, thereby enhancing its sparse regularization capability.
However, computing the full loss $ \mathcal{L}_{\mathrm{PE}} $ exactly is computationally expensive because it requires evaluating the network over all $ T $ transformations for each sample.
Moreover, the resulting heavy redundancy among these transformed patches can degrade learning efficiency of PE-CSNet.
This remains true even with a minimal batch size of $ 1 $.
To tackle this, we introduce a stochastic training scheme.
Instead of evaluating the full equivariant loss $ \mathcal{L}_\textrm{EQ}(\Theta) $ in \eqref{7}, in each iteration we approximate it by an unbiased estimator:
\begin{equation}\label{17}
\widetilde{\mathcal{L}}_\textrm{EQ}(\Theta) := \dfrac{1}{S\widetilde{T}}\sum_{i=1}^S\sum_{k=1}^{\widetilde{T}}L\left(\mathcal{T}_{j^k}\bds{x}_i, f_\textrm{PE}(\Phi(\mathcal{T}_{j^k}\bds{x}_i); \Theta)\right),
\end{equation}
where the indices $ \{j^k\}_{k=1}^{\widetilde{T}} $ are drawn uniformly at random from $ \{1,2,\ldots,T\} $.
By construction, $ \mathbb{E}\left[\widetilde{\mathcal{L}}_\textrm{EQ}\right] = \mathcal{L}_\textrm{EQ} $.
This leads to a stochastic version of the total loss:
\begin{equation}\label{18}
\widetilde{\mathcal{L}}_\textrm{PE}(\Theta) = \mathcal{L}_\textrm{DC}(\Theta) + \beta\widetilde{\mathcal{L}}_\textrm{EQ},
\end{equation}
which is also an unbiased estimator of $ \mathcal{L}_\textrm{PE} $, i.e., $ \mathbb{E}\left[\widetilde{\mathcal{L}}_\textrm{PE}\right] = \mathcal{L}_\textrm{PE} $.
The network is then updated using the gradient of this stochastic loss, thereby reducing the computational cost in each iteration from $ O(T) $ to $ O(\widetilde{T}) $.
For the optimization itself, we employ the Adam optimizer \cite{KDP14} to train PE-CSNet $ f_\textrm{PE}(\cdot;\Theta) $ on the dataset $\mathcal{S} = \{(\bds{y}_i,\bds{x}_i)\}_{i=1}^S$, initializing the weights with Xavier scheme \cite{GXB10} (see Subsection \ref{S421} for more training details).
After $ t $ training epochs, we obtain the final model $ f_\textrm{PE}(\cdot;\widehat{\Theta}) $.
For reconstruction, given undersampled measurement $ \bds{y} $, the corresponding ground truth image is recovered via Algorithm \ref{PE-CSNet}.
This algorithm describes the forward pass of the trained PE-CSNet $ f_\textrm{PE}(\cdot;\widehat{\Theta}) $.
Its structure follows a natural data flow: within each iteration $ n $, the current approximation $ \bds{x}^{(n)} $ is first processed by the auxiliary patch update module $ (\bds{Z}^{(n)}_{ij})\; (1\leq i,j\leq N_p) $ (using parameters $ \alpha^{(n)}, \bds{\theta}^{(n)} $) to produce the auxiliary $ \bds{z}^{(n)} $; this is then used to compute the next approximation $ \bds{x}^{(n+1)} $ via the reconstruction module $ (\bds{X}^{(n+1)}) $ (using the penalty parameter $ \rho^{(n+1)} $).
Here, $ N_p $ is the number of patches per image side ($ N_p = M/m $), consistent with its definition in Section~\ref{S1}.
See Figure \ref{F2} for visualizing the architecture of PE-CSNet.

\begin{algorithm}[htbp]
\caption{Inference of PE-CSNet}\label{PE-CSNet}

\textbf{Input: }$ \Phi $, $\bds{y}$, $\mathcal{P}_r$, $N$, $ N_p $, $\widehat{\Theta} = \{\alpha^{(n)}, \bds{\theta}^{(n)}, \rho^{(n+1)}\}_{n=0}^{N-1}$ \Comment{$ N_p $: side length of the patch grid}

\textbf{Output:} an approximation for the ground truth image $\bds{x}$

\textbf{Initialize:} $n = 0$, $\bds{x}^{(0)} := \Phi^H\bds{y}$ \Comment{$D_{ij}$ are predefined for $1\leq i,j\leq N_p$}

\begin{algorithmic}[1]
\item \textbf{while} $n<N$ \textbf{do}
\State \quad Compute $ \mathcal{U}_{\alpha,\bds{\theta}}^{(n)} = I + \alpha^{(n)}\mathcal{U}(\cdot;\bds{\theta}^{(n)}) $
\State \quad $ D_{ij}\bds{z}^{(n)} = \mathcal{U}_{\alpha,\bds{\theta}}^{(n)}\left(\mathcal{P}_r\left(D_{ij}\bds{x}^{(n)}\right)\right) $ \Comment{auxiliary patch update module $ (\bds{Z}^{(n)}_{ij}) $}
\State \quad $\bds{z}^{(n)} = \mathtt{Reassemble}(\{D_{ij}\bds{z}^{(n)}\}) $
\State \quad $\bds{x}^{(n+1)} = \left(\Phi^H\Phi + \rho^{(n+1)}I\right)^{-1}\left(\Phi^H \bds{y} + \rho^{(n+1)}\bds{z}^{(n)}\right) $ \Comment{reconstruction module $ (\bds{X}^{(n+1)}) $}
\State \quad $ n\gets n+1 $
\item \textbf{end while}
\item Set final approximation to be $\bds{x}^{(N)}$.
\end{algorithmic}
\end{algorithm}

\subsection{Convergence}
This subsection discusses the convergence of a simplified, parameter-shared version of PE-CSNet.
In this version, the network parameters $\bds{\theta}^{(n)}$ in Algorithm \ref{PE-CSNet} are equal across all stages; that is, $ \bds{\theta}^{(n)} = \bds{\theta} $ for all $ n $.
The scaling parameters $ \{\alpha^{(n)}\} $ and penalty parameters $ \{\rho^{(n)}\} $ remain stage-dependent.
For the purpose of establishing a well-posed convergence theory, we regard the resulting iteration of the simplified PE-CSNet as an infinite process $ (n=0,1,2,\ldots) $ in this subsection.
We emphasize that in our practical implementation and experiments, we use the full (non-shared) version of PE-CSNet with finite depth (i.e., Algorithm \ref{PE-CSNet}), which leads to better performance, as demonstrated in Subsection \ref{S44}.
The convergence result relies on the following assumptions.

\begin{assumption}[Lipschitz continuity of the shared network]\label{A1}
In the simplified (parameter-shared) version, the learned network $ \mathcal{U}(\cdot;\bds{\theta}) $ is Lipschitz continuous; i.e., there exists a constant $ 0<L<\infty $ such that
\begin{displaymath}
\|\mathcal{U}(\bds{x};\bds{\theta}) - \mathcal{U}(\bds{y};\bds{\theta})\|_2 \leq L\|\bds{x} - \bds{y}\|_2,\quad \forall \bds{x}, \bds{y}.
\end{displaymath}
\end{assumption}

\begin{remark}\label{R2}
The Lipschitz condition in Assumption \ref{A1} is standard.
It is justified because the network $ \mathcal{U}(\cdot;\bds{\theta}) $ is a finite-depth U-Net composed of layers with bounded operators (e.g., convolutional weights, $ 1 $-Lipschitz activation, and fixed batch normalization).
Since the learned parameter $ \bds{\theta} $ is fixed and bounded, the operator norms of all layers are bounded, ensuring the overall network is Lipschitz continuous.
\end{remark}

\begin{assumption}\label{A2}
In the simplified (parameter-shared) version of PE-CSNet, the learned parameters $ \alpha^{(n)}\; (n = 0,1,2,\ldots) $ and $\rho^{(n)}\; (n = 1,2,3,\ldots) $ satisfy:
\begin{displaymath}
\alpha^{(0)}\ge1,\quad \rho^{(1)}\ge1,\quad \rho^{(n+1)}\geq \gamma \rho^{(n)},\quad \alpha^{(n)}\leq \dfrac{1}{\rho^{(n)}},\qquad n = 1,2,3,\ldots,
\end{displaymath}
where $\gamma>1$ is a constant.
\end{assumption}

\begin{remark}\label{R3}
The conditions in Assumption \ref{A2} are mild and can be enforced during training, for instance, by projecting the parameters onto the prescribed intervals.
Since we use the more expressive stage-specific (non-shared) parameters in experiments, we do not enforce these projection constraints.
This assumption is introduced for the convergence analysis of the simplified, parameter-shared version of PE-CSNet.
\end{remark}

Based on the above assumptions, our main result (i.e., Theorem \ref{thm}) is that the simplified, parameter-shared version of PE-CSNet, viewed as an iterative algorithm, converges linearly to a fixed point.
The detailed proof is provided in Appendix~\ref{A}.
Note that the convergence discussed here refers to the network's iterative output (i.e., $ \bds{x}^{(n)} $ in Algorithm \ref{PE-CSNet}), not to the training of network parameters.

\begin{theorem}\label{thm}
Under Assumptions \ref{A1} and \ref{A2} (where $\gamma>1$), the infinite sequence $\{\bds{x}^{(n)}\}_{n=0}^{\infty}$ generated by the simplified, parameter-shared version of PE-CSNet (i.e., $ \bds{\theta}^{(n)} = \bds{\theta} $ for all $ n $ in Algorithm \ref{PE-CSNet}) converges linearly to a point $ \bds{x}^* $; that is, the convergence error satisfies
\begin{equation*}
\|\bds{x}^{(n)} - \bds{x}^*\|_2\leq C\gamma^{-n},
\end{equation*}
where $ C $ is a positive constant independent of $ n $.
\end{theorem}

\begin{remark}\label{R5}\rm
Theorem \ref{thm} guarantees convergence of the sequence generated by the simplified, parameter-shared PE-CSNet to a fixed point.
Whether this fixed point coincides with a critical point of the original model \eqref{9} depends on properties of the learned network $ \mathcal{U}(\cdot;\bds{\theta}) $ that are not assumed here.
\end{remark}

\begin{remark}\label{R6}\rm
Theorem \ref{thm} establishes that there exists a (parameter-shared) instance of the proposed PE-CSNet architecture (i.e., Algorithm \ref{PE-CSNet}) whose iterations converge linearly.
This demonstrates that PE-CSNet itself is inherently capable of producing convergent sequences, providing a theoretical basis for its practical stability.
The empirical convergence results presented in Subsections \ref{S421} and \ref{S431} provide support for this theoretical insight.
For the comparison of PE-CSNet and its parameter-shared version, see Subsection \ref{S44}.
\end{remark}

\section{Numerical experiments}\label{S4}
\setcounter{equation}{0}

This section presents comprehensive numerical experiments to demonstrate the effectiveness and data efficiency of the proposed PE-CSNet.
We begin by detailing the experimental setup in Subsection~\ref{S41}.
Then, in Subsections~\ref{S42} and~\ref{S43}, we evaluate PE-CSNet for compressive sensing MRI (CS-MRI) and compressive sensing coded diffraction patterns (CS-CDP), respectively.
The results show that the proposed PE-CSNet outperforms state-of-the-art CS methods in reconstruction quality and demonstrates stable empirical convergence, while maintaining fast computational speed.
Notably, with only 100 training samples, PE-CSNet already achieves competitive performance, highlighting its data efficiency (see Subsections~\ref{S42} and~\ref{S43}).
Ablation studies (Subsection~\ref{S44}) verify the importance of its stage-specific parameters and patch-based a priori information.

\subsection{Experimental setup}\label{S41}

All deep learning models, including PE-CSNet and the compared deep-learning baselines, are implemented in PyTorch and trained on the same NVIDIA A100 GPU under Linux.
The traditional (non-learning) algorithms are executed on a CPU (Intel Core i7-10700 (2.90GHz), 32 GB RAM, Windows 11).

We adopt three widely used metrics to quantify CS reconstruction quality: the normalized root mean square error (NRMSE), the peak signal-to-noise ratio (PSNR), and the structural similarity index (SSIM).
For a ground truth image $ \bds{x} $ and its reconstruction $ \widetilde{\bds{x}} $, the NRMSE is defined as
\[
\mathrm{NRMSE}(\widetilde{\bds{x}}, \bds{x}) := \dfrac{\|\bds{x} - \widetilde{\bds{x}}\|_2}{\|\bds{x}\|_2}.
\]
Definitions of PSNR and SSIM follow the standard formulations given in \cite{HAZ10}.
Lower NRMSE and higher PSNR/SSIM values indicate better reconstruction.

Datasets details, sampling patterns, and training hyperparameters for each specific task are provided in the corresponding subsections below.

\subsection{Compressive sensing MRI experiments}\label{S42}

Compressive sensing magnetic resonance imaging (CS-MRI) aims to reconstruct high-quality images from sub-Nyquist $ k $-space measurements (i.e., Fourier domain), thereby accelerating data acquisition.
In CS-MRI, the measurement matrix is a partial Fourier transform matrix, i.e., $ \Phi = \bds{U}\bds{F} $, which is a structurally random matrix (SRM, see Remark \ref{R1}).
Here, $ \bds{U} $ is an undersampling matrix and $ \bds{F} $ is a 2D discrete Fourier transform matrix (DFT).
In our experiments, the undersampling matrix $ \bds{U} $ is implemented as either a 1D Cartesian sampling mask or a 2D random sampling mask, as illustrated in Figure \ref{F3}, where white pixels denote sampled $ k $-space points and black pixels denote unsampled points.
To evaluate the proposed PE-CSNet on this task, we first demonstrate the reconstruction capability and empirical convergence behavior of PE-CSNet (Subsection~\ref{S421}), followed by a comparison with baseline methods, including both state-of-the-art deep unrolling networks and a well-established traditional algorithm (Subsection~\ref{S422}).

\begin{figure}[!htbp]
\centering
\includegraphics[width=.6\linewidth]{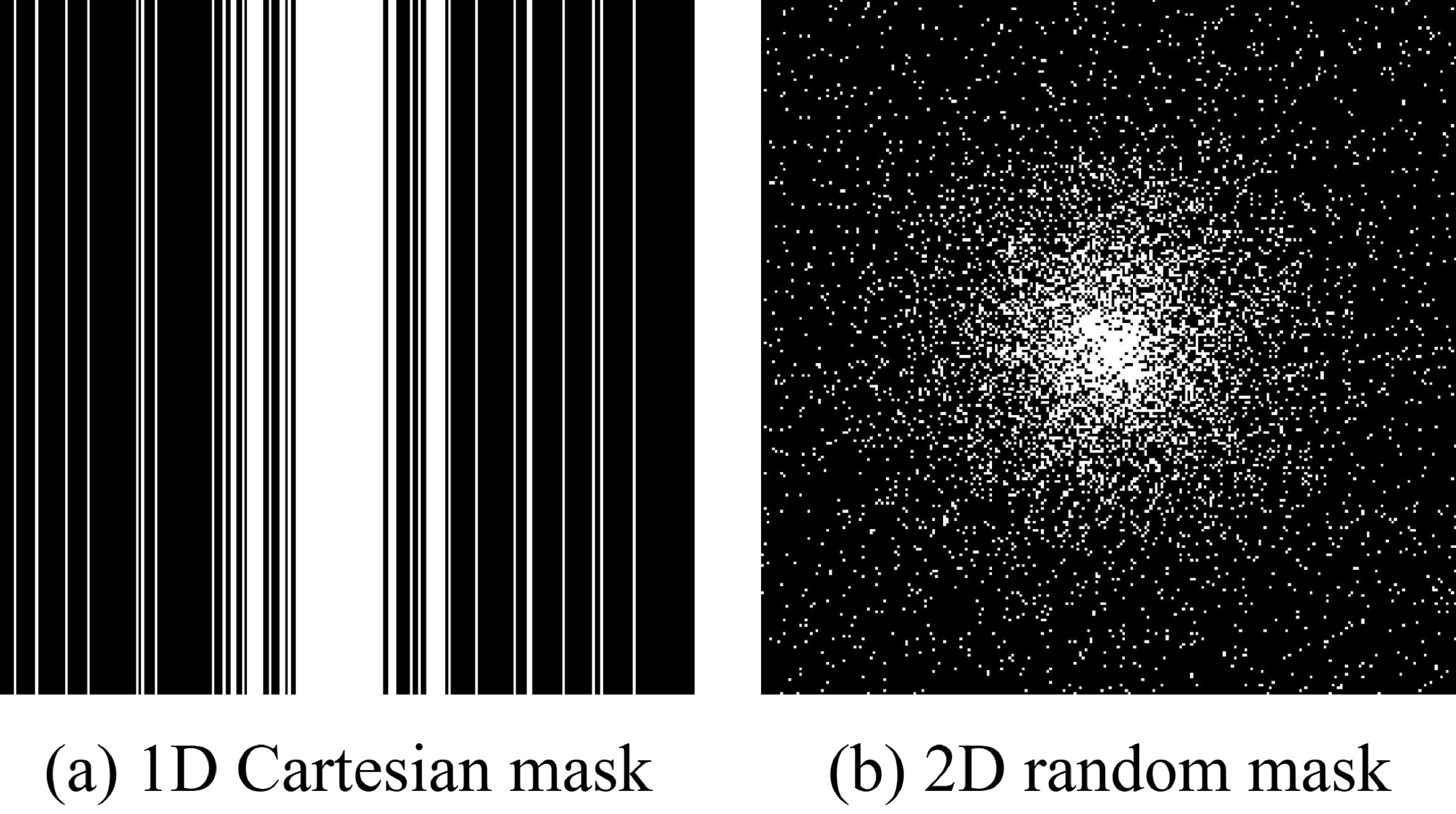}
\caption{Visualization of the undersampling masks used in CS-MRI experiments.
White and black pixels represent sampled and unsampled $ k $-space locations, respectively.
(a) 1D Cartesian mask; (b) 2D random mask.}\label{F3}
\end{figure}

\subsubsection{Performance of PE-CSNet on compressive sensing MRI}\label{S421}
This subsection presents the reconstruction results and empirical convergence behavior of the trained PE-CSNet on the CS-MRI task.
The experimental setup is detailed as follows:
\begin{itemize}
\item \textbf{Dataset and sampling patterns.} We use a standard MRI dataset\footnote{\url{https://github.com/yangyan92/Pytorch_ADMM-CSNet}} \cite{YYS20} which consists of 200 fully sampled, complex-valued brain magnetic resonance (MR) images of size $ (M\times M) $ with $ M=256 $ (i.e., the image dimension defined in Section~\ref{S1}), evenly split into $ 100 $ training and $ 100 $ test images.
We evaluate two sampling patterns: 1D Cartesian masks at sampling rates of 20\%, 30\%, 40\%, and 2D random masks at sampling rates of 5\%, 10\%, 20\% (see Figure \ref{F3}).
A separate instance of PE-CSNet is trained for each mask-rate combination.

\item \textbf{Network configuration.} The proposed (non-shared) PE-CSNet is unrolled for $ N=4 $ stages and operates on an $ (N_p\times N_p) $ grid of image patches with $ N_p = 8 $.
Given the image size $ M=256 $, this yields a patch size of $ m = M/N_p = 32 $.
Here, $ N_p $ and $ m $ are consistent with their definition in Section~\ref{S1}.
The radius of projection operator $ \mathcal{P}_r $ (see \eqref{19}) is set to $ r=100 $.
The measurement matrix in PE-CSNet is defined as $ \Phi = \bds{U}\bds{F} $, where $ \bds{F} $ is the 2D DFT matrix, and $ \bds{U} $ is the undersampling matrix corresponding to the specific mask and sampling rate in the CS-MRI task.

\item \textbf{Training details.} The proposed PE-CSNet is trained end-to-end using the stochastic equivariant training strategy described in Section~\ref{S31}. The optimization adopts the Adam optimizer \cite{KDP14} (learning rate $ 10^{-4} $, batch size $ 1 $, total epochs $ 6000 $) with the Xavier initialization \cite{GXB10}, to minimize the stochastic loss $ \widetilde{\mathcal{L}}_\textrm{PE} $ (see \eqref{18}), where the trade-off parameter $ \beta = 1 $ and $ \widetilde{T}=8 $.
The model is trained for the full $ 6000 $ epochs, and the final parameters (represented by the trained network $ f_\textrm{PE}(\cdot;\widehat{\Theta}) $ as described in Subsection~\ref{S31}) are used for testing.
\end{itemize}
The setup described above applies specifically to the CS-MRI experiments.

\textbf{Reconstruction capacity.} We begin with a quantitative investigation of its reconstruction capability under various sampling patterns and sampling rates: 1D Cartesian masks at 20\%, 30\%, and 40\% sampling rates, and 2D random masks at 5\%, 10\%, and 20\% sampling rates.
The average reconstruction quality (NRMSE, PSNR, SSIM) of PE-CSNet on the test set is summarized in the bottom rows of Tables \ref{T1} and \ref{T2}.
Specifically, the bottom row of Table \ref{T1} shows the results for the networks trained with 1D Cartesian masks at 20\%, 30\%, and 40\% sampling rates,
while the bottom row of Table \ref{T2} shows the results for the networks trained with 2D random masks at 5\%, 10\%, and 20\% sampling rates.
Representative visual reconstructions by PE-CSNet for the CS-MRI task are presented in Figures \ref{F4}(a) and \ref{F4}(b).
Specifically, the first row of Figure \ref{F4}(a) shows the magnitude images of the reconstruction results for 1D Cartesian masks at 20\%, 30\%, and 40\% sampling rates, and the corresponding ground truth images; the first row of Figure \ref{F4}(b) shows the magnitude images of the reconstruction results for 2D random masks at 5\%, 10\%, and 20\% sampling rates, and the corresponding ground truth images.
In all these images, a red box highlights a region of interest, whose enlarged view is provided for detailed examination.
Immediately below each reconstruction, the corresponding pixel-wise error map is displayed. This map visualizes the magnitude (absolute value) of the complex difference between the reconstruction and the ground truth, where brighter pixels indicate larger errors.
Reconstruction results in bottom rows of Tables \ref{T1} and \ref{T2}, and Figure \ref{F4} show that the proposed PE-CSNet achieves satisfactory reconstruction capability across both sampling patterns with different sampling rates.
Notably, it maintains reliable performance even at the low sampling rates (e.g., 20\% for Cartesian and 5\% for random masks), as supported by the quantitative metrics and visual results.
We believe this is because our PE-CSNet has learned the patch-based a priori information of the ground truth image, which is critical for CS-MRI.

\textbf{Empirical convergence.} We further investigate the convergence behavior of (non-shared) PE-CSNet.
This analysis is performed for both sampling patterns, as summarized in Figures \ref{F5} (1D Cartesian) and \ref{F6} (2D random).
Specifically, Figures \ref{F5}(a) and \ref{F6}(a) show the stage-wise reconstruction outputs of PE-CSNet for a 1D Cartesian mask (40\% sampling rate) and a 2D random mask (20\% sampling rate), respectively.
Each subfigure displays the magnitude images of the intermediate outputs $ \bds{x}^{(n)}\; (n=0,1,\ldots,4) $ followed by the ground truth, with the corresponding PSNR (between the output and the ground truth) annotated below each intermediate output of PE-CSNet.
The quantitative convergence curves are plotted in Figures \ref{F5}(b) and \ref{F6}(b).
Figure \ref{F5}(b) shows the NRMSE of the intermediate outputs $ \bds{x}^{(n)}\; (n = 0,1,\ldots,4) $ of PE-CSNet as a function of the stage number $ n $ for a representative test sample.
The curves correspond to reconstructions of this sample under 1D Cartesian masks with sampling rates of 20\%, 30\%, and 40\%.
Similarly, Figure \ref{F6}(b) displays the NRMSE convergence curves of PE-CSNet for a representative test sample.
These curves correspond to the reconstructions of this sample under 2D random masks with sampling rates of 5\%, 10\%, and 20\%.
The convergence behavior of PE-CSNet, averaged over the entire test set, is further illustrated in Figures \ref{F5}(c) and \ref{F6}(c).
Figure \ref{F5}(c) shows the mean NRMSE of the intermediate outputs $ \bds{x}^{(n)}\; (n = 0,1,\ldots,4) $ of PE-CSNet, along with the $ \pm 1 $ standard deviation band (shaded region), computed across the entire test set for the 1D Cartesian mask with a sampling rate of 30\%.
Similarly, Figure \ref{F6}(c) presents the mean NRMSE and the associated $ \pm 1 $ standard deviation band (shaded region) for PE-CSNet across the entire test set, corresponding to the 2D random mask with a sampling rate of 10\%.
Although we only derive the linear convergence rate of a simplified, parameter-shared version of PE-CSNet (see Theorem \ref{thm}), the stable empirical convergence of the full (non-shared) version (see Figures \ref{F5} and \ref{F6}) aligns with the theoretical insight.
The convergence of the parameter-shared version, as a special instance of the full model, provides theoretical support for the observed empirical behavior.
This consistency suggests the inherent stability of our PE-CSNet framework, as discussed in Remark \ref{R6}.

\begin{figure}[!htbp]
\centering
\includegraphics[width=\linewidth]{image/performance_MRI.eps}
\caption{Reconstruction results of PE-CSNet on CS-MRI.
(a) Results for 1D Cartesian masks at 20\%, 30\%, and 40\% sampling rates.
(b) Results for 2D random masks at 5\%, 10\%, and 20\% sampling rates.
For each case, the magnitude image of the reconstruction, the corresponding error map (magnitude of the complex difference), a detailed crop, and the magnitude image of the ground truth are displayed.
}\label{F4}
\end{figure}

\begin{figure}[!htbp]
\centering
\includegraphics[width=\linewidth]{image/Convergence_Cartesian.eps}
\caption{Empirical convergence analysis of PE-CSNet on CS-MRI for 1D Cartesian masks.
(a) Stage-wise reconstruction outputs for the 1D Cartesian mask at 40\% sampling rate.
Magnitude images of the intermediate outputs $ \bds{x}^{(n)}\; (n = 0,1,\ldots,4) $ are shown, followed by the ground truth.
Corresponding PSNR values are annotated below each output.
(b) NRMSE convergence curves for a representative sample under 1D Cartesian masks at 20\%, 30\%, and 40\% sampling rates.
(c) Mean NRMSE $ \pm 1 $ standard deviation (shaded region) across the entire test set for the 1D Cartesian mask at the sampling rate of 30\%.
}\label{F5}
\end{figure}

\begin{figure}[!htbp]
\centering
\includegraphics[width=\linewidth]{image/Convergence_random.eps}
\caption{Empirical convergence analysis of PE-CSNet on CS-MRI for 2D random masks.
(a) Stage-wise reconstruction outputs for the 2D random mask at 20\% sampling rate.
Magnitude images of the intermediate outputs $ \bds{x}^{(n)}\; (n = 0,1,\ldots,4) $ are shown, followed by the ground truth.
Corresponding PSNR values are annotated below each output.
(b) NRMSE convergence curves for a representative sample under 2D random masks at 5\%, 10\%, and 20\% sampling rates.
(c) Mean NRMSE $ \pm 1 $ standard deviation (shaded region) across the entire test set for the 2D random mask at the sampling rate of 10\%.
}\label{F6}
\end{figure}

\subsubsection{Comparison with related works}\label{S422}
In this subsection, we compare the performance of our PE-CSNet (the same networks used in Subsection \ref{S421}) with several popular CS methods.
The methods selected for comparison include Zero-filling \cite{BMA01}, PANO \cite{QXH14}, ISTA-Net \cite{ZJG18}, and ADMM-CSNet \cite{YYS20}.
For a fair comparison, the reproduced reconstruction methods of PANO, ISTA-Net (9 stages), and ADMM-CSNet (10 stages) are implemented with their hyperparameters (e.g., learning rates, regularization parameters) strictly following the configurations in their respective original publications \cite{QXH14, ZJG18, YYS20}.
We make a quantitative comparison of their reconstruction capability under various sampling patterns and sampling rates: 1D Cartesian mask at 20\%, 30\%, and 40\% sampling rates, and 2D random masks at 5\%, 10\%, 20\% sampling rates.
The average reconstruction quality (NRMSE, PSNR, SSIM) and the inference times per sample of the above reconstruction methods are summarized in Tables \ref{T1} and \ref{T2}.
In these two tables, the best performance in each column (i.e., for each metric and sampling rate) is highlighted in bold.
Note that the inference times for the reconstruction methods (except PANO) are measured on the CPU/GPU as detailed in Subsection \ref{S41}; the CPU inference time for PANO (approximately 3 minutes per sample) is omitted from the tables as it is orders of magnitude larger than the millisecond-scale times of the other methods, which would impede a clear visual comparison.
Specifically, Table \ref{T1} shows the results for the CS methods using 1D Cartesian masks at 20\%, 30\%, and 40\% sampling rates, while Table \ref{T2} shows the results for the CS methods using 2D random masks at 5\%, 10\%, 20\% sampling rates.
In addition to the quantitative metrics, we compare the visual quality of the reconstructions.
Figure \ref{F7} presents the magnitude images of the reconstructions of different CS methods under the 1D Cartesian mask at a 30\% sampling rate.
Each reconstruction is labeled with the corresponding method name, and the PSNR (in dB) with respect to the ground truth is annotated beneath it.
In all these images, a red box highlights a region of interest, whose enlarged view is provided for detailed comparison.
As shown in Tables \ref{T1} and \ref{T2}, and Figures \ref{F7}, the proposed PE-CSNet consistently outperforms the state-of-the-art unrolling networks (i.e., ISTA-Net and ADMM-CSNet) and other traditional CS methods across both sampling patterns with different sampling rates, while also maintaining a fast computational speed.
Notably, the proposed PE-CSNet, with $ N=4 $ unrolled stages, outperforms both ISTA-Net (9 stages) and ADMM-CSNet (10 stages) with fewer stages, as supported by the quantitative metrics and visual results.
We attribute this effective performance to two key factors: the equivariant training strategy (details in Section \ref{S31}) and the patch-based architecture of PE-CSNet.
These designs collectively enhance the effective diversity of the training data.
This, in turn, enables the network to more effectively learn the patch-based a priori information inside the ground truth image, which is crucial for high-quality CS-MRI reconstruction.

\begin{table}[!htpb]
\caption{Comparison of average reconstruction quality (NRMSE ($\times 10^{-2}$), PSNR in dB, SSIM ($\times 10^{-2}$)) and average inference times per sample on CS-MRI for 1D Cartesian masks with sampling rates of 20\%, 30\%, and 40\%.
Best performance in each column is highlighted in bold.}\label{T1}
\scriptsize
\renewcommand\arraystretch{1.2}
\centering
\begin{tabular}{p{2.5cm}p{1.1cm}<{\centering}p{.8cm}<{\centering}p{.8cm}<{\centering}p{1.1cm}<{\centering}p{.8cm}<{\centering}p{.8cm}<{\centering}p{1.1cm}<{\centering}p{.8cm}<{\centering}p{.8cm}<{\centering}p{1.5cm}<{\centering}}
\hline
\multirow{2}*{Method} & \multicolumn{3}{c}{20\%} & \multicolumn{3}{c}{30\%} & \multicolumn{3}{c}{40\%} &
Time \\
\cline{2-11}
& NRMSE & PSNR & SSIM & NRMSE & PSNR & SSIM & NRMSE & PSNR & SSIM & CPU/GPU\\
\hline
Zero-filling \cite{BMA01} &{27.43} &{25.93} &{61.93} &{21.68} &{28.10} &{70.38} &{20.01} &{28.62} &{70.96} &{2ms/-}\\

PANO \cite{QXH14} &{20.24} &{28.55} &{73.78} &{16.17} &{30.48} &{80.25} &{12.12} &{33.02} &{83.48} &{-/-}\\

ISTA-Net \cite{ZJG18} &{16.64} &{31.70} &{80.59} &{11.21} &{34.39} &{86.71} &{8.16} &{37.10} &{90.65} & {-/8ms}\\

ADMM-CSNet \cite{YYS20} & {15.98} & {31.38} & {81.15} & {11.19} & {34.40} & {87.50} & {8.08} & {37.27} & {90.76} & {-/458ms}\\

PE-CSNet & \textbf{14.69} & \textbf{31.87} & \textbf{81.68} & \textbf{10.33} & \textbf{34.94} & \textbf{88.02} & \textbf{7.82} & \textbf{37.53} & \textbf{91.44} & {87ms/7ms}\\
\hline
\end{tabular}
\end{table}

\begin{table}[!htbp]
\caption{Comparison of average reconstruction quality (NRMSE ($ \times 10^{-2} $), PSNR (dB), SSIM ($ \times 10^{-2} $)) and average inference times per sample on CS-MRI for 2D random masks with sampling rates of 5\%, 10\%, and 20\%.
Best performance in each column is highlighted in bold.}\label{T2}
\scriptsize
\renewcommand\arraystretch{1.2}
\centering
\begin{tabular}{p{2.5cm}p{1.1cm}<{\centering}p{.8cm}<{\centering}p{.8cm}<{\centering}p{1.1cm}<{\centering}p{.8cm}<{\centering}p{.8cm}<{\centering}p{1.1cm}<{\centering}p{.8cm}<{\centering}p{.8cm}<{\centering}p{1.5cm}<{\centering}}
\hline
\multirow{2}*{Method} & \multicolumn{3}{c}{5\%} & \multicolumn{3}{c}{10\%} & \multicolumn{3}{c}{20\%} &
Time \\
\cline{2-11}
& NRMSE & PSNR & SSIM & NRMSE & PSNR & SSIM & NRMSE & PSNR & SSIM & CPU/GPU\\
\hline
Zero-filling \cite{BMA01} &{34.34} &{23.95} &{51.11} &{29.83} &{25.37} &{56.29} &{25.73} &{26.73} &{60.45} &{2ms/-}\\

PANO \cite{QXH14} &{24.44} &{27.21} &{66.68} &{19.79} &{29.12} &{72.34} &{15.11} &{31.22} &{76.91} &{-/-}\\

ISTA-Net \cite{ZJG18} &{19.81} &{29.70} &{70.20} &{15.61} &{31.69} &{76.45} &{10.54} &{34.52} &{82.77} &{-/8ms}\\

ADMM-CSNet \cite{YYS20} & {20.07} & {29.52} & {73.40} & {15.04} & {31.90} & {78.16} & {10.53} & {34.73} & {83.47} & {-/458ms}\\

PE-CSNet & \textbf{18.24} & \textbf{30.19} & \textbf{75.54} & \textbf{13.91} & \textbf{32.33} & \textbf{79.65} & \textbf{10.02} & \textbf{35.04} & \textbf{84.86} & {82ms/7ms}\\
\hline
\end{tabular}
\end{table}

\begin{figure}[!htbp]
\centering
\includegraphics[width=\linewidth]{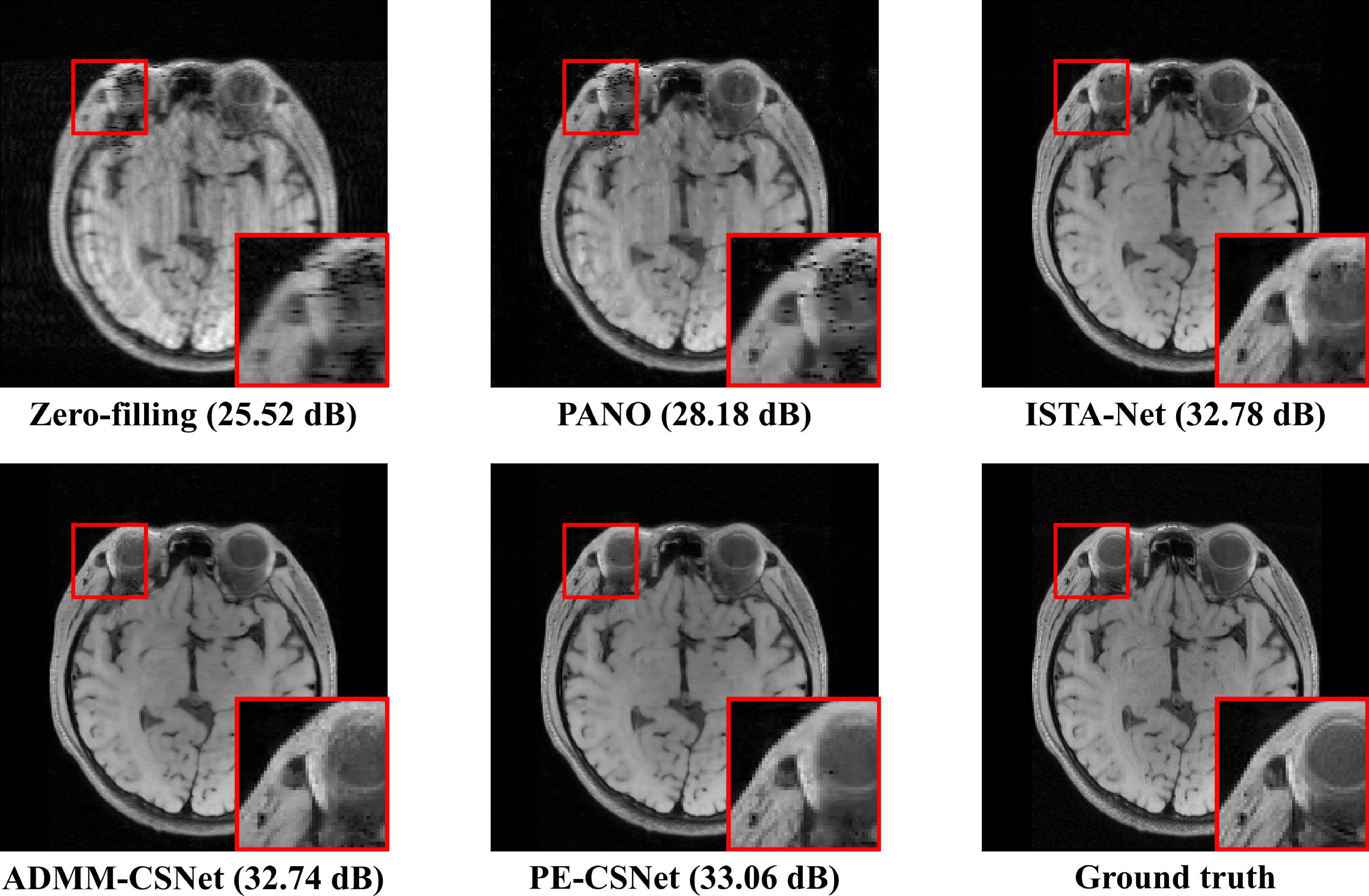}
\caption{Visual comparison of reconstructed magnitude images (and corresponding zoomed-in views) from different CS-MRI methods under a 1D Cartesian mask (30\% sampling rate).
Red boxes indicate the regions of interest for detailed visual comparison.
}\label{F7}
\end{figure}


\subsection{Compressive sensing CDP experiments}\label{S43}

In this work, we consider the compressive sensing problem using coded diffraction patterns (CS-CDPs), which aims to reconstruct a real-valued image from sub-Nyquist samples of its Fourier domain, acquired through random spatial encoding \cite{YYS20, LHL25}.
In the CS-CDP, the measurement matrix is a structurally random matrix (see Remark \ref{R1}) with the definition $ \Phi = \bds{U}\bds{F}\bds{S} $, where $ \bds{U} $ is a diagonal undersampling matrix whose diagonal entries are independent Bernoulli $ \{0,1\} $ random variables, $ \bds{F} $ is the 2D discrete Fourier transform matrix (DFT), and $ \bds{S} $ is a diagonal sign matrix whose diagonal entries are complex numbers of unit modulus (e.g., random phase).
To evaluate the proposed PE-CSNet on this task, we first demonstrate the reconstruction capability and empirical convergence behavior of PE-CSNet (Subsection \ref{S431}), followed by a comparison with representative deep unrolling methods (Subsection \ref{S432}).

\subsubsection{Performance of PE-CSNet on compressive sensing CDP}\label{S431}

This subsection presents the reconstruction results and empirical convergence behavior of the trained PE-CSNet on the CS-CDP task.
The experimental setup is detailed as follows:
\begin{itemize}
\item \textbf{Dataset and sampling patterns.} The training dataset consists of 100 images from the Berkeley BSD dataset\footnote{\url{https://www2.eecs.berkeley.edu/Research/Projects/CS/vision/bsds/}} \cite{MDF01};
the testing dataset is composed of $10$ standard testing images\footnote{\url{https://github.com/cszn/IRCNN/tree/master/testsets}} \cite{ZKZ17}. 
All images from both the training and test set are resized to $ (M\times M) $ with $ M=256 $ (i.e., the image dimension defined in Section~\ref{S1}).
We evaluate the undersampling masks at sampling rates of 5\%, 10\%, 20\%.
A separate instance of PE-CSNet is trained for each sampling rate.

\item \textbf{Network configuration.} The proposed (non-shared) PE-CSNet adopts the same configuration as in Subsection \ref{S421}, except that PE-CSNet is unrolled for $ N=6 $ stages.
In addition, the measurement matrix within the network is defined as $ \Phi = \bds{U}\bds{F}\bds{S} $ (see the beginning of Subsection \ref{S43}).

\item \textbf{Training details.} The training details of the proposed PE-CSNet are the same as in Subsection \ref{S421}.
\end{itemize}
The setup described above applies specifically to the CS-CDP experiments.


\textbf{Reconstruction capacity.} We first quantitatively evaluate the reconstruction capability of PE-CSNet under sampling rates of 5\%, 10\%, and 20\%.
The average reconstruction quality (NRMSE, PSNR, SSIM) of PE-CSNet on the test set is summarized in the bottom row of Table \ref{T3}.
Representative visual reconstructions by PE-CSNet for the CS-CDP task are presented in Figure \ref{F10}(a) and \ref{F10}(b).
Specifically, the first rows of Figures \ref{F10}(a) and \ref{F10}(b) show the reconstructions at sampling rates of 5\%, 10\%, and 20\%, and the corresponding ground truth images.
In all these images, a red box highlights a region of interest, whose enlarged view is provided for detailed examination.
Immediately below each reconstruction, the corresponding pixel-wise error map is displayed.
This map visualizes the absolute difference between the reconstruction and the ground truth, where brighter pixels indicate larger errors.
Reconstruction results in the bottom row of Table \ref{T3} and Figure \ref{F10} show that the proposed PE-CSNet achieves satisfactory reconstruction capability under different sampling rates.
Notably, it maintains reliable performance even at low sampling rates (e.g., 5\% and 10\%), as supported by the quantitative metrics and visual results.
This satisfactory performance can be attributed to the ability of PE-CSNet to learn the patch-based a priori information inherent in the ground truth image, which is critical for solving CS-CDP.

\textbf{Empirical convergence.} We further investigate the convergence behavior of (non-shared) PE-CSNet.
Figure \ref{F11}(a) shows the stage-wise reconstruction outputs of PE-CSNet for a sampling rate of 20\%.
Each subfigure displays the intermediate outputs $ \bds{x}^{(n)}\; (n=0,1,\ldots,6) $ followed by the ground truth, with the corresponding PSNR (between the output and the ground truth) annotated below each intermediate output of PE-CSNet.
The quantitative convergence curves are plotted in Figure \ref{F11}(b).
Figure \ref{F11}(b) shows the NRMSE of the intermediate outputs $ \bds{x}^{(n)}\; (n = 0,1,\ldots,6) $ of PE-CSNet as a function of the stage number $ n $ for a representative test sample.
The curves correspond to reconstructions of this sample under sampling rates of 5\%, 10\%, and 20\%.
The convergence behavior of PE-CSNet, averaged over the entire test set, is further illustrated in Figure \ref{F11}(c).
Figure \ref{F11}(c) shows the mean NRMSE of the intermediate outputs $ \bds{x}^{(n)}\; (n = 0,1,\ldots,6) $ of PE-CSNet, along with the $ \pm 1 $ standard deviation band (shaded region), computed across the entire test set for the sampling rate of 10\%.
Although we only derive the linear convergence rate of a simplified, parameter-shared version of PE-CSNet (see Theorem \ref{thm}), the stable empirical convergence of the full (non-shared) version (see Figure \ref{F11}) aligns with the theoretical insight.
The convergence of the parameter-shared version, as a special instance of the full model, provides theoretical support for the observed empirical behavior.
This consistency suggests the inherent stability of our PE-CSNet framework, as discussed in Remark \ref{R6}.

\subsubsection{Comparison with related works}\label{S432}
In this subsection, we compare the performance of our PE-CSNet (the same networks used in Subsection \ref{S431}) with several popular CS methods.
The methods selected for comparison include ISTA-Net \cite{ZJG18} and ADMM-CSNet \cite{YYS20}.
For a fair comparison, the reproduced reconstruction methods of ISTA-Net (9 stages) and ADMM-CSNet (10 stages) are implemented with their hyperparameters (e.g., learning rates, epochs) strictly following the configurations in their respective original publications \cite{ZJG18, YYS20}.
We make a quantitative comparison of their reconstruction capability under sampling rates of 5\%, 10\%, 20\%.
The average reconstruction quality (NRMSE, PSNR, SSIM) and the inference times per sample of the above reconstruction methods are summarized in Table \ref{T3}.
In this table, the best performance in each column (i.e., for each metric and sampling rate) is highlighted in bold.
Note that the inference times for the reconstruction methods are measured on the GPU as detailed in Subsection \ref{S41}.
In addition to the quantitative metrics, we compare the visual quality of the reconstructions.
Figure \ref{F12} presents the reconstructions of different CS methods at sampling rates of 5\%, 10\%, and 20\%.
Each reconstruction is labeled with the corresponding PSNR (in dB) with respect to the ground truth.
In all these images, a red box highlights a region of interest, whose enlarged view is provided for detailed comparison.
As shown in Table \ref{T3} and Figure \ref{F12}, the proposed PE-CSNet not only consistently outperforms the state-of-the-art unrolling networks (i.e., ISTA-Net and ADMM-CSNet) across different sampling rates, but also maintains a fast computational speed.
Notably, PE-CSNet, with only $ N=6 $ unrolled stages, outperforms both ISTA-Net (9 stages) and ADMM-CSNet (10 stages) with fewer stages, as supported by the quantitative metrics and visual results.
We attribute this effective performance to two key factors: the equivariant training strategy (details in Section \ref{S31}) and the patch-based architecture.
These designs collectively enhance the effective diversity of the training data.
This, in turn, enables the network to more effectively learn the patch-based a priori information inside the ground truth image, which is crucial for high-quality CS-CDP reconstruction.

\begin{figure}[!htbp]
\centering
\includegraphics[width=\linewidth]{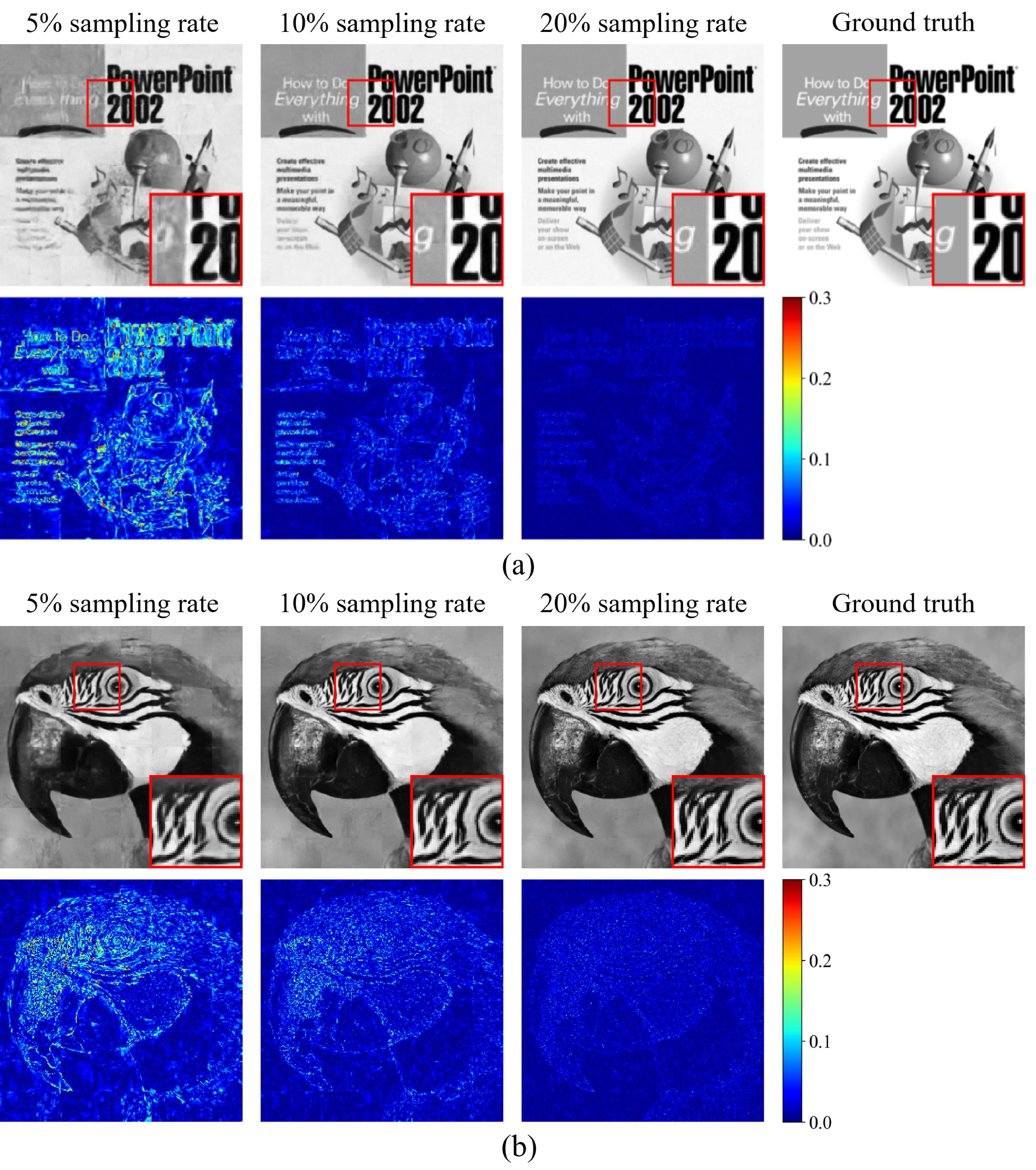}
\caption{Reconstruction results of PE-CSNet on CS-CDP at 5\%, 10\%, and 20\% sampling rates.
For each reconstruction, the corresponding error map, a detailed crop, and the corresponding ground truth are displayed.
}\label{F10}
\end{figure}

\begin{figure}[!htbp]
\centering
\includegraphics[width=\linewidth]{image/Convergence_CDP.eps}
\caption{Empirical convergence analysis of PE-CSNet for CS-CDP.
(a) Stage-wise reconstruction outputs for a 20\% sampling rate.
The intermediate outputs $ \bds{x}^{(n)}\; (n = 0,1,\ldots,6) $ are shown, followed by the ground truth.
Corresponding PSNR values are annotated below each output.
(b) NRMSE convergence curves for a representative sample under 5\%, 10\%, and 20\% sampling rates.
(c) Mean NRMSE $ \pm 1 $ standard deviation (shaded region) across the entire test set for the sampling rate of 10\%.
}\label{F11}
\end{figure}

\begin{table}[!htpb]
\caption{Comparison of average reconstruction quality (NRMSE ($ \times 10^{-2} $), PSNR (dB), and SSIM ($\times 10^{-2}$)) and average inference times per sample on CS-CDP for sampling rates of 5\%, 10\%, and 20\%.
Best performance in each column is highlighted in bold.}\label{T3}
\scriptsize
\renewcommand\arraystretch{1.2}
\centering
\begin{tabular}{p{2.5cm}p{1.1cm}<{\centering}p{.8cm}<{\centering}p{.8cm}<{\centering}p{1.1cm}<{\centering}p{.8cm}<{\centering}p{.8cm}<{\centering}p{1.1cm}<{\centering}p{.8cm}<{\centering}p{.8cm}<{\centering}p{1.5cm}<{\centering}}
\hline
\multirow{2}*{Method} & \multicolumn{3}{c}{5\%} & \multicolumn{3}{c}{10\%} & \multicolumn{3}{c}{20\%} &
Time \\
\cline{2-11}
& NRMSE & PSNR & SSIM & NRMSE & PSNR & SSIM & NRMSE & PSNR & SSIM & GPU\\
\hline

ISTA-Net \cite{ZJG18} & {15.41} & {21.42} & {55.02} & {7.72} & {27.43} & {80.34} & {4.38} & {32.43} & {90.54} & {7ms}\\

ADMM-CSNet \cite{YYS20} & {11.70} & {23.78} & {63.07} & {6.79} & {28.55} & {79.06} & {3.76} & {33.87} & {90.43} & {301ms}\\

PE-CSNet & \textbf{8.34} & \textbf{26.74} & \textbf{78.15} & \textbf{5.10} & \textbf{31.10} & \textbf{89.08} & \textbf{2.80} & \textbf{36.50} & \textbf{95.49} & 11ms \\
\hline
\end{tabular}

\end{table}

\begin{figure}[!htbp]
\centering
\includegraphics[width=\linewidth]{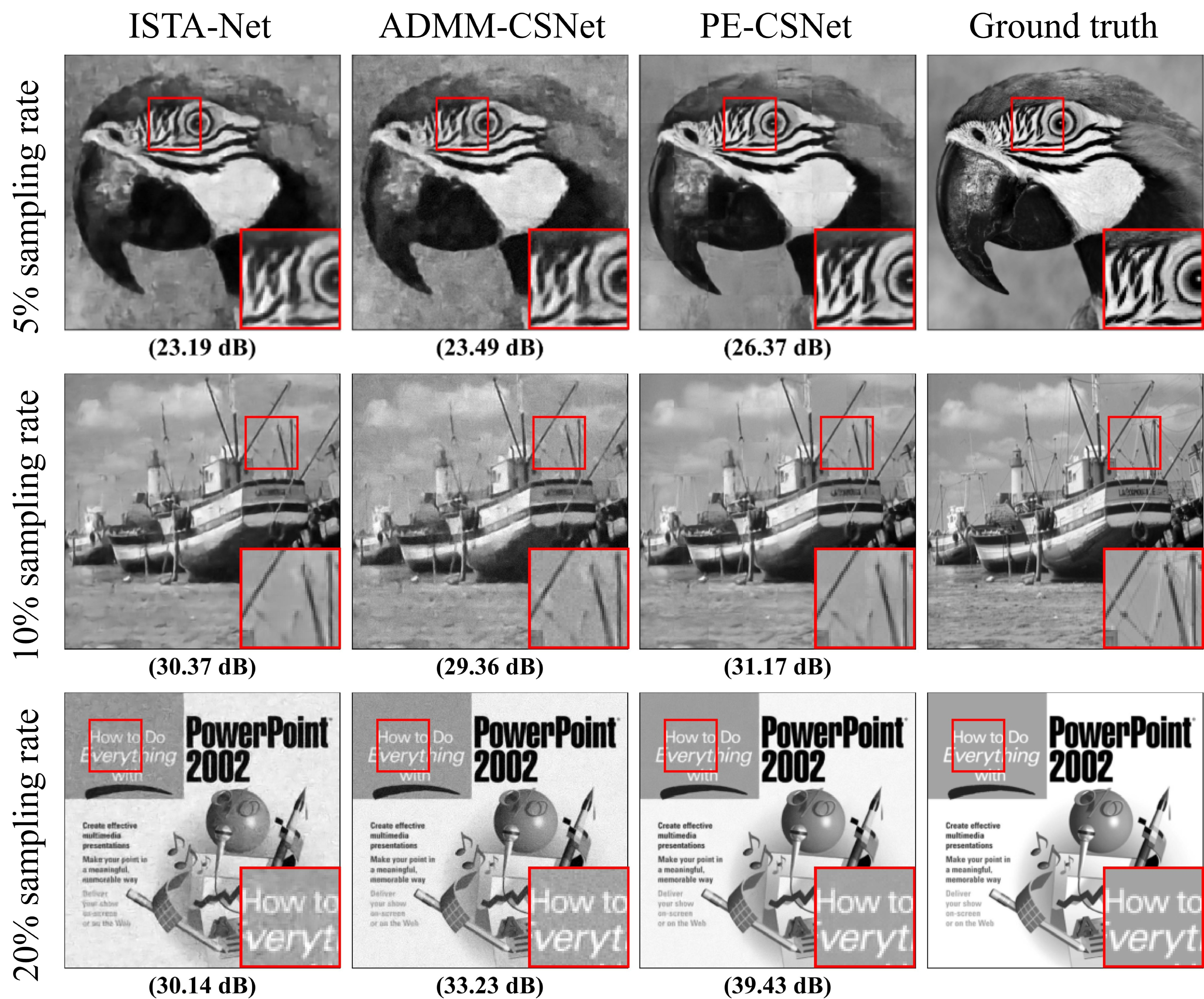}
\caption{Visual comparison of reconstructions (and corresponding zoomed-in views) from different CS-CDP methods under sampling rates of 5\%, 10\%, and 20\%.
Red boxes indicate the regions of interest for detailed visual comparison.
}\label{F12}
\end{figure}

\subsection{Ablation study}\label{S44}

To verify the individual contribution of the two core designs in the proposed PE-CSNet, namely the stage-specific (non-shared) parameters and the patch-based a priori information, we conduct ablation studies by comparing the full version of PE-CSNet (i.e., Algorithm \ref{PE-CSNet}) with the following two ablated versions of PE-CSNet:
\begin{itemize}
\item \textbf{PE-CSNet (shared)}: A version where the network parameters $\bds{\theta}^{(n)}$ are shared across all unrolled stages, i.e., $\bds{\theta}^{(n)} = \bds{\theta}$ for all $n$ in Algorithm \ref{PE-CSNet}.
This version ablates the stage-specific adaptation of PE-CSNet.

\item \textbf{PE-CSNet (global)}: A version where the patch-extraction operator $ D_{ij} $ and the operator $ \mathtt{Reassemble}(\cdot) $ are removed, and the core networks $ \mathcal{U}_{\alpha,\bds{\theta}}^{(n)}\; (n=0,1,\ldots,N-1) $ are applied directly to the full image.
That is, lines 3--4 of Algorithm \ref{PE-CSNet} are replaced by
\[\bds{z}^{(n)} = \mathcal{U}_{\alpha,\bds{\theta}}^{(n)}\left(\mathcal{P}_r\left(\bds{x}^{(n)}\right)\right).\]
This version ablates the explicit patch-based structure of PE-CSNet.
\end{itemize}
For a fair comparison, we evaluate the above three models on the same CS-MRI task under 1D Cartesian masks at sampling rates of 20\%, 30\%, and 40\% (see Subsection \ref{S42}).
For the experimental setup, all these models follow the dataset and sampling patterns, network configuration, and training details as in Subsection \ref{S421}.
We make a quantitative comparison of their reconstruction capacity under 1D Cartesian masks at 20\%, 30\%, and 40\% sampling rates.
The average reconstruction quality, measured by NRMSE, PSNR, and SSIM, is summarized in Table \ref{T4}.
In this table, the best performance in each column (i.e., for each metric and sampling rate) is highlighted in bold.
As shown in Table \ref{T4}, 
the proposed PE-CSNet consistently outperforms its two ablated versions (i.e., PE-CSNet (shared) and PE-CSNet (global)) under 1D Cartesian masks across different sampling rates.
Although Theorem \ref{thm} establishes the linear convergence rate for the parameter-shared version of PE-CSNet (i.e., PE-CSNet (shared)), our ablation study reveals a performance degradation in this version.
This indicates that sharing parameters across stages, while theoretically convenient, limits the network's capacity to learn stage-specific features, thereby restricting its refinement capability across iterations.
Similarly, we also observe a clear performance drop of PE-CSNet (global) version in the ablation study, which confirms the critical role of the patch-based architecture of the proposed PE-CSNet.
On the one hand, the patch-based architecture allows the network to explicitly utilize the self-similarity inside the ground truth images, and thus learn the patch-based a priori information, which is crucial for high-quality CS recovery.
On the other hand, this patch-based architecture, together with the equivariant training strategy (see Subsection \ref{S31}), substantially increases the diversity of the training data.
This enhanced diversity serves as a powerful regularizer and enables robust training of our PE-CSNet.
Therefore, the patch-based architecture of PE-CSNet is essential both as a source of important a priori information and as a mechanism for enhancing learning.

\begin{table}[!htpb]
\caption{Comparisons of average reconstruction quality (NRMSE ($ \times 10^{-2} $), PSNR (dB), and SSIM ($ \times 10^{-2} $)) and average inference times per sample on CS-MRI for 1D Cartesian masks at 20\%, 30\%, and 40\% sampling rates.
Best performance in each column is highlighted in bold.}\label{T4}
\scriptsize
\renewcommand\arraystretch{1.3}
\centering
\begin{tabular}{p{2.6cm}p{1.1cm}<{\centering}p{.8cm}<{\centering}p{.8cm}<{\centering}p{1.1cm}<{\centering}p{.8cm}<{\centering}p{.8cm}<{\centering}p{1.1cm}<{\centering}p{.8cm}<{\centering}p{.8cm}<{\centering}}
\hline
\multirow{2}*{Method} & \multicolumn{3}{c}{20\%} & \multicolumn{3}{c}{30\%} & \multicolumn{3}{c}{40\%} \\
\cline{2-10}
& NRMSE & PSNR & SSIM & NRMSE & PSNR & SSIM & NRMSE & PSNR & SSIM\\
\hline

PE-CSNet (shared) &{17.01} &{30.68} &{78.97} &{12.00} &{33.47} &{84.85} &{8.84} &{36.35} &{89.02}\\

PE-CSNet (global) & {16.35} & {31.14} & {80.04} & {11.40} & {34.07} & {85.72} & {8.31} & {36.94} & {89.89}\\

PE-CSNet & \textbf{14.69} & \textbf{31.87} & \textbf{81.68} & \textbf{10.33} & \textbf{34.94} & \textbf{88.02} & \textbf{7.82} & \textbf{37.53} & \textbf{91.44}\\
\hline
\end{tabular}

\end{table}


\section{Conclusion}\label{S5}
In this paper, we proposed PE-CSNet, a novel deep unrolling network for compressive sensing (CS) recovery.
It is derived by unrolling the block coordinate descent (BCD) algorithm applied to a generalized patch-based CS model. Through end-to-end training, PE-CSNet learns to solve various CS tasks (e.g., CS-MRI or CS-CDP) and automatically learns the patch-based sparsity transforms and other hyperparameters.
To ensure robust learning with limited data, we introduced a stochastic equivariant strategy that leverages the network's patch-based architecture.
We also studied a simplified parameter-shared PE-CSNet iteration and proved
that, under suitable parameter constraints, the generated sequence converges
linearly to a fixed point. This result provides a stability guarantee for a
restricted instance of the proposed architecture.

Extensive numerical experiments demonstrate that our PE-CSNet achieves state-of-the-art performance for various CS tasks, including CS-MRI and CS-CDP, while maintaining fast inference speed.
First, we observe that PE-CSNet not only generates reliable reconstructions for these tasks across various sampling patterns and rates (see Subsections \ref{S421} and \ref{S431}), but also consistently outperforms existing state-of-the-art unrolling networks (e.g., ISTA-Net, ADMM-CSNet) and some traditional CS methods (see Subsections \ref{S422} and \ref{S432}).
This superior performance is attributed to the network's ability to effectively learn and exploit the patch-based a priori information inherent in the ground truth images.
The equivariant training strategy (see Subsection \ref{S31}) and the patch-based architecture work together to enhance the diversity of the training data, which helps the network learn the patch-based a priori information more effectively.
Second, PE-CSNet exhibits stable empirical convergence across all tested conditions for both CS-MRI and CS-CDP tasks (see Subsections \ref{S421} and \ref{S431}).
This observed convergence behavior in practice is well-aligned with our theoretical finding that a parameter-shared instance of the PE-CSNet architecture (i.e., PE-CSNet (shared) in Subsection \ref{S44}) guarantees linear convergence.
This consistency between theory and practice suggests the inherent stability of our PE-CSNet framework.
Furthermore, the ablation study verifies the necessity of the stage-specific design.
Specifically, the comparison between PE-CSNet and PE-CSNet (shared) confirms that using stage-specific parameters is important.
This design enables the network to learn distinct features at each iteration, thereby enhancing its refinement capability across stages.
Third, the ablation study independently validates the critical role of the patch-based architecture of PE-CSNet.
The clear performance drop of the PE-CSNet (global) version (see Subsection \ref{S44}) demonstrates that this architecture is indispensable.
It functions in two ways: as an explicit patch-based structure that enables PE-CSNet to capture image self-similarity; and as a data-augmentation mechanism (jointly with the equivariant training strategy) that greatly increases the training sample diversity.
Thus, the patch-based architecture of PE-CSNet is essential both as a source of significant a priori information and as a mechanism for enhancing learning.
However, it is also noteworthy that in our experiments, the reconstruction quality of PE-CSNet degrades when the sampling rate is small.
One plausible reason for this could be the relatively limited size of the training dataset (100 samples), which may constrain the network's ability to generalize effectively under extreme measurement sparsity.
It is also interesting to explore unsupervised or self-supervised strategies to learn the underlying patch-based transform sparsity, which will be considered as a future work.

\section*{Acknowledgments}

The work of K. Li and Z. Zhou is partially supported by National Natural Science Foundation of China (Projects 12422117),
Hong Kong Research Grants Council (15302323 and 12426312) and an internal grant of Hong Kong Polytechnic University (Project ID: P0053938, Work Programme: 4-ZZVA).
The work of B. Zhang and H. Zhang is partially supported by the NNSF of China grant (12431016 and 12271515).

\appendix
\renewcommand{\thesection}{\Alph{section}}
\renewcommand{\theequation}{\thesection.\arabic{equation}}
\setcounter{equation}{0}
\section{Proof of Theorem \ref{thm}}\label{A}

\begin{proof}
We first show that the iteration sequence $ \{\bds{x}^{(n)}\} $ in Algorithm \ref{PE-CSNet} is a Cauchy sequence, i.e.,
\begin{equation}\label{23}
\sum_{n=0}^\infty \left\|\bds{x}^{(n+1)}-\bds{x}^{(n)}\right\|_2 < \infty.
\end{equation}
In fact, by the triangle inequality, we have
\begin{equation}\label{24}
\sum_{n=0}^\infty \left\|\bds{x}^{(n+1)}-\bds{x}^{(n)}\right\|_2
\le \sum_{n=0}^\infty \left\|\bds{x}^{(n+1)}-\bds{z}^{(n)}\right\|_2 + \sum_{n=0}^\infty \left\|\bds{z}^{(n)}-\bds{v}^{(n)}\right\|_2 + \sum_{n=0}^\infty \left\|\bds{v}^{(n)}-\bds{x}^{(n)}\right\|_2,
\end{equation}
where $ \bds{v}^{(n)} := \mathtt{Reassemble}\left(\left\{\mathcal{P}_r(D_{ij} \bds{x}^{(n)})\right\}\right) $.
Here, $ \mathtt{Reassemble}(\cdot) $ and $ D_{ij} $ follow Section \ref{S1}, and $ \mathcal{P}_r(\cdot) $ is the same as in \eqref{19}.
Thus, it suffices to show that the last three summation terms in \eqref{24} are finite.

First, we show that $ \sum_{n=0}^\infty \left\|\bds{z}^{(n)}-\bds{v}^{(n)}\right\|_2< \infty $.
Due to the $L$-Lipschitz continuity of $\mathcal{U}(\cdot;\bds{\theta})$ (Assumption \ref{A1}), the following holds for all $\bds{x}\in B_r$.
\begin{equation}\label{22}
\|\mathcal{U}(\bds{x};\bds{\theta})\|_2 \le \|\mathcal{U}(\bds{0};\bds{\theta})\|_2 + Lr.
\end{equation}
Recalling $ D_{ij}\bds{z}^{(n)} := \mathcal{U}_{\alpha,\bds{\theta}}^{(n)}\left(\mathcal{P}_r\left(D_{ij}\bds{x}^{(n)}\right)\right) $, we have
\begin{equation}\label{21}
\left\|\bds{z}^{(n)}-\bds{v}^{(n)}\right\|_2^2
=\left(\alpha^{(n)}\right)^2\sum_{i,j}\left\|\mathcal{U}\left(\mathcal{P}_r\left(D_{ij}\bds{x}^{(n)}\right); \bds{\theta}\right)\right\|_2^2
\le C_1\left(\alpha^{(n)}\right)^2
\end{equation}
for some constant $C_1>0$.
By Assumption \ref{A2} ($\alpha^{(n)} \le 1/\rho^{(n)}$ and $\rho^{(n+1)} \ge \gamma \rho^{(n)}$ with $\gamma > 1$), it follows that
\begin{equation}\label{25}
\sum_{n=1}^\infty \left\|\bds{z}^{(n)}-\bds{v}^{(n)}\right\|_2\le \sqrt{C_1}\sum_{n=1}^\infty\alpha^{(n)}\le \sqrt{C_1}\sum_{n=1}^\infty\dfrac{1}{\rho^{(n)}}<\infty.
\end{equation}

Second, we show that $ \sum_{n=0}^\infty \left\|\bds{x}^{(n+1)}-\bds{z}^{(n)}\right\|_2<\infty $.
From the optimality condition of \eqref{12}, we have $\nabla f(\bds{x}^{(n+1)}) + \rho^{(n+1)}(\bds{x}^{(n+1)}-\bds{z}^{(n)}) = 0$ with $f(\bds{x}) = \frac{1}{2}\|\Phi \bds{x} - \bds{y}\|_2^2$, which yields
\begin{equation}\label{27}
\left\|\bds{x}^{(n+1)}-\bds{z}^{(n)}\right\|_2 = \dfrac{\left\|\nabla f(\bds{x}^{(n+1)})\right\|_2}{\rho^{(n+1)}}.
\end{equation}
Since $\{\bds{z}^{(n)}\}$ and $\{\bds{x}^{(n)}\}$ are bounded (ensured by \eqref{21}, \eqref{14}, and the bounded norm of $(\Phi^H\Phi + \rho^{(n+1)}I)^{-1}$), the gradient $\left\{\nabla f(\bds{x}^{(n)})\right\} $ is uniformly bounded by $C_2>0$, i.e., 
$\left\|\nabla f(\bds{x}^{(n)})\right\|_2\le C_2$. Thus we obtain,
\begin{equation}\label{28}
\sum_{n=0}^\infty\left\|\bds{x}^{(n+1)}-\bds{z}^{(n)}\right\|_2 \le C_2\sum_{n=0}^\infty\frac{1}{\rho^{(n+1)}}<\infty.
\end{equation}

Third, we show that $ \sum_{n=0}^\infty \left\|\bds{v}^{(n)}-\bds{x}^{(n)}\right\|_2<\infty $.
Let $\bds{\delta}_{ij}^{(n)} := \left\|D_{ij}\bds{v}^{(n)}-D_{ij}\bds{x}^{(n)}\right\|_2$.
Using the non-expansive property of the projection operator $\mathcal{P}_r$, we obtain
\begin{equation*}
\begin{aligned}
\bds{\delta}_{ij}^{(n+1)}
&\le \left\|D_{ij}\bds{v}^{(n+1)}-\mathcal{P}_r(D_{ij}\bds{z}^{(n)})\right\|_2+\left\|\mathcal{P}_r(D_{ij}\bds{z}^{(n)})-D_{ij}\bds{z}^{(n)}\right\|_2 + \left\|D_{ij}\bds{z}^{(n)} - D_{ij}\bds{x}^{(n+1)}\right\|_2\\
&\le 2\left\|D_{ij}\bds{x}^{(n+1)}-D_{ij}\bds{z}^{(n)}\right\|_2 + \left\|D_{ij}\bds{z}^{(n)}-D_{ij}\bds{v}^{(n)}\right\|_2.
\end{aligned}
\end{equation*}
Combining this with \eqref{21}, \eqref{27}, and the Cauchy-Schwarz inequality leads to
\begin{equation}\label{31}
\begin{aligned}
\left\|\bds{v}^{(n+1)}-\bds{x}^{(n+1)}\right\|_2^2
&\le 2N_p^2\left(4\sum_{i,j}\left\|D_{ij}\bds{x}^{(n+1)}-D_{ij}\bds{z}^{(n)}\right\|_2^2 + \sum_{i,j}\left\|D_{ij}\bds{z}^{(n)}-D_{ij}\bds{v}^{(n)}\right\|_2^2\right)\\
&=2N_p^2\left(4\left\|\bds{x}^{(n+1)}-\bds{z}^{(n)}\right\|_2^2 + \left\|\bds{z}^{(n)}-\bds{v}^{(n)}\right\|_2^2\right)\le \dfrac{C_3}{\left(\rho^{(n)}\right)^2}
\end{aligned}
\end{equation}
for some constant $ C_3>0 $, where $ N_p $ is the patch grid size as in Section \ref{S1}.
This implies
\begin{equation}\label{30}
\sum_{n=2}^\infty \left\|\bds{v}^{(n)}-\bds{x}^{(n)}\right\|_2\le \sqrt{C_3}\sum_{n=2}^\infty\dfrac{1}{\rho^{(n-1)}}<\infty.
\end{equation}
Combining \eqref{24}, \eqref{25}, \eqref{28}, and \eqref{30} proves \eqref{23}, meaning $ \left\{\bds{x}^{(n)}\right\} $ converges to some $\bds{x}^*$.

It remains to estimate the convergence rate of $ \left\{\bds{x}^{(n)}\right\} $.
In fact, by \eqref{21}, \eqref{27}, \eqref{31}, for all $ k\ge 2 $
\begin{equation}\label{32}
\left\|\bds{x}^{(k+1)}-\bds{x}^{(k)}\right\|_2\le \dfrac{C_2}{\rho^{(k+1)}} + \dfrac{\sqrt{C_1}}{\rho^{(k)}} + \dfrac{\sqrt{C_3}}{\rho^{(k-1)}}\le \dfrac{C_4}{\rho^{(k-1)}}
\end{equation}
for some constant $ C_4>0 $.
Using $ \rho^{(k)}\ge \gamma^{k-1}\rho^{(1)} $ (Assumption \ref{A2}), the error rate for $ n\ge 3 $ is bounded by
\begin{equation*}
\left\|\bds{x}^{(n)}-\bds{x}^*\right\|_2 \le \sum_{k=n}^\infty \left\|\bds{x}^{(k+1)}-\bds{x}^{(k)}\right\|_2\le \sum_{k=n}^\infty\dfrac{C_4}{\rho^{(k-1)}}\le \dfrac{C_4}{\rho^{(1)}} \sum_{k=n}^\infty\dfrac{1}{\gamma^{k-2}} = \dfrac{C_4\gamma^3}{\rho^{(1)}(\gamma-1)}\gamma^{-n}.
\end{equation*}
Note that $ \gamma>1 $ (see Assumption \ref{A2}).
The proof is complete.
\end{proof}


\begin{thebibliography}{99}
\providecommand{\url}[1]{\texttt{#1}}
\providecommand{\urlprefix}{URL }
\expandafter\ifx\csname urlstyle\endcsname\relax
\providecommand{\doi}[1]{doi:\discretionary{}{}{}#1}\else
\providecommand{\doi}{doi:\discretionary{}{}{}\begingroup
\urlstyle{rm}\Url}\fi

\bibitem{ASM19}
S.~Arridge, P.~Maass, O.~Öktem, and C.-B.~Schönlieb, Solving inverse problems using data-driven models, \textit{Acta Numer.} \textbf{28} (2019), 1--174.


\bibitem{BAT09}
A.~Beck and M.~Teboulle, A fast iterative shrinkage-thresholding algorithm for
linear inverse problems, \textit{SIAM J. Imaging Sci.} \textbf{2} (2009),
183--202.

\bibitem{BAT13}
A.~Beck and L.~Tetruashvili, On the convergence of block coordinate descent
type methods, \textit{SIAM J. Optim.} \textbf{23} (2013), 2037--2060.

\bibitem{BMB18}
M.~Benning and M.~Burger, Modern regularization methods for inverse problems, \textit{Acta Numer.} \textbf{27} (2018), 1--111.

\bibitem{BMA01}
M.~A. Bernstein, S.~B. Fain, and S.~J. Riederer, Effect of windowing and
zero-filled reconstruction of {MRI} data on spatial resolution and
acquisition strategy, \textit{J. Magn. Reson. Imaging} \textbf{14} (2001),
270--280.

\bibitem{BAC10}
A.~Buades, B.~Coll, and J.~M. Morel, Image denoising methods. {A} new nonlocal
principle, \textit{SIAM Rev.} \textbf{52} (2010), 113--147.

\bibitem{CEJ06}
E.~J. Candes and T.~Tao, Near-optimal signal recovery from random projections:
universal encoding strategies?, \textit{IEEE Trans. Inform. Theory}
\textbf{52} (2006), 5406--5425.

\bibitem{CEJ08}
E.~J. Candes and M.~B. Wakin, An introduction to compressive sampling,
\textit{IEEE Signal Process. Mag.} \textbf{25} (2008), 21--30.

\bibitem{CJL20}
J.~Cao, S.~Liu, H.~Liu, and H.~Lu, {CS-MRI} reconstruction based on analysis
dictionary learning and manifold structure regularization, \textit{Neural
Networks} \textbf{123} (2020), 217--233.

\bibitem{CAP11}
A.~Chambolle and T.~Pock, A first-order primal-dual algorithm for convex
problems with applications to imaging, \textit{J. Math. Imaging Vision}
\textbf{40} (2011), 120--145.

\bibitem{CBZ25}
B.~Chen and J.~Zhang, Practical compact deep compressed sensing, \textit{IEEE
Trans. Pattern Anal. Mach. Intell.} \textbf{47} (2025), 1610--1626.

\bibitem{CDT21}
D.~Chen, J.~Tachella, and M.~E. Davies, Equivariant imaging: Learning beyond
the range space, \textit{2021 IEEE/CVF International Conference on Computer
Vision (ICCV 2021)}, 2021, 4359--4368.

\bibitem{CDT22}
D.~Chen, J.~Tachella, and M.~E. Davies, Robust equivariant imaging: a fully
unsupervised framework for learning to image from noisy and partial
measurements, \textit{2022 IEEE/CVF Conference on Computer Vision and Pattern
Recognition (CVPR 2022)}, 2022, 5637--5646.

\bibitem{CDD23}
D.~Chen, M.~Davies, M.~J. Ehrhardt, C.-B. Schonlieb, F.~Sherry, and
J.~Tachella, Imaging with equivariant deep learning: From unrolled network
design to fully unsupervised learning, \textit{IEEE Signal Process. Mag.}
\textbf{40} (2023), 134--147.

\bibitem{CTS16}
T.~S. Cohen and M.~Welling, Group equivariant convolutional networks, M.~F.
Balcan and K.~Q. Weinberger (eds.), \textit{International Conference on
Machine Learning, Vol 48}, 2016.

\bibitem{CZX24}
Z.-X. Cui, Q.~Zhu, J.~Cheng, B.~Zhang, and D.~Liang, Deep unfolding as
iterative regularization for imaging inverse problems, \textit{Inverse
Problems} \textbf{40} (2024), Paper No. 025011, 31.

\bibitem{DKF07}
K.~Dabov, A.~Foi, V.~Katkovnik, and K.~Egiazarian, Image denoising by sparse
3-{D} transform-domain collaborative filtering, \textit{IEEE Trans. Image
Process.} \textbf{16} (2007), 2080--2095.

\bibitem{DID04}
I.~Daubechies, M.~Defrise, and C.~De~Mol, An iterative thresholding algorithm
for linear inverse problems with a sparsity constraint, \textit{Comm. Pure
Appl. Math.} \textbf{57} (2004), 1413--1457.

\bibitem{DTT12}
T.~T. Do, L.~Gan, N.~H. Nguyen, and T.~D. Tran, Fast and efficient compressive
sensing using structurally random matrices, \textit{IEEE Trans. Signal
Process.} \textbf{60} (2012), 139--154.

\bibitem{DWS14}
W.~Dong, G.~Shi, X.~Li, Y.~Ma, and F.~Huang, Compressive sensing via nonlocal
low-rank regularization, \textit{IEEE Trans. Image Process.} \textbf{23}
(2014), 3618--3632.

\bibitem{DDL06}
D.~L. Donoho, Compressed sensing, \textit{IEEE Trans. Inf. Theory} \textbf{52}
(2006), 1289--1306.

\bibitem{EMP19}
M.~P. Edgar, G.~M. Gibson, and M.~J. Padgett, Principles and prospects for
single-pixel imaging, \textit{Nat Photonics} \textbf{13} (2019), 13--20.

\bibitem{FZH21}
Z.~Fabian, R.~Heckel, and M.~Soltanolkotabi, Data augmentation for deep
learning based accelerated {MRI} reconstruction with limited data, M.~Meila
and T.~Zhang (eds.), \textit{International Conference on Machine Learning},
vol. 139, 2021.

\bibitem{FJX24}
J.~Fu, Q.~Xie, D.~Meng, and Z.~Xu, Rotation equivariant proximal operator for
deep unfolding methods in image restoration, \textit{IEEE Trans. Pattern
Anal. Mach. Intell.} \textbf{46} (2024), 6577--6593.

\bibitem{GXB10}
X.~Glorot and Y.~Bengio, Understanding the difficulty of training deep
feedforward neural networks, Y.~W. Teh and M.~Titterington (eds.),
\textit{Proceedings of the Thirteenth International Conference on Artificial
Intelligence and Statistics}, vol.~9, 2010, 249--256.

\bibitem{GKL10}
K.~Gregor and Y.~LeCun, Learning fast approximations of sparse coding,
\textit{Proceedings of the 27th International Conference on International
Conference on Machine Learning}, 2010, 399--406.

\bibitem{HLC09}
L.~He and L.~Carin, Exploiting structure in wavelet-based {Bayesian}
compressive sensing, \textit{IEEE Trans. Signal Process.} \textbf{57} (2009),
3488--3497.

\bibitem{HAZ10}
A.~Horé and D.~Ziou, Image quality metrics: {PSNR} vs. {SSIM}, \textit{2010
20th International Conference on Pattern Recognition}, 2010, 2366--2369.

\bibitem{HZN21}
Z.~Hu, F.~Nie, R.~Wang, and X.~Li, Low rank regularization: A review,
\textit{Neural Networks} \textbf{136} (2021), 218--232.

\bibitem{KDP14}
D.~P. Kingma and J.~Ba, Adam: A method for stochastic optimization,
\textit{arXiv:1412.6980}  (2014).

\bibitem{KFB11}
F.~Knoll, K.~Bredies, T.~Pock, and R.~Stollberger, Second order total
generalized variation ({TGV}) for {MRI}, \textit{Magn. Reson. Med.}
\textbf{65} (2011), 480--491.

\bibitem{KAS17}
A.~Krizhevsky, I.~Sutskever, and G.~E. Hinton, {ImageNet} classification with
deep convolutional neural networks, \textit{Commun. ACM} \textbf{60} (2017),
84--90.

\bibitem{LKV15}
K.~Lenc and A.~Vedaldi, Understanding image representations by measuring their
equivariance and equivalence, \textit{2015 IEEE Conference on Computer Vision
and Pattern Recognition (CVPR)}, 2015, 991--999.

\bibitem{LCY09}
C.~Li, W.~Yin, and Y.~Zhang, User’s guide for {TVAL3}: {TV} minimization by
augmented lagrangian and alternating direction algorithms, \textit{CAAM Rep.}
\textbf{20} (2009), 4.

\bibitem{LHL25}
H.~Li and J.~Li, Truncated amplitude flow with coded diffraction patterns,
\textit{Inverse Problems} \textbf{41} (2025), Paper No. 015002, 30.

\bibitem{LZZ24}
K.~Li, B.~Zhang, and H.~Zhang, Reconstruction of inhomogeneous media by an
iteration algorithm with a learned projector, \textit{Inverse Problems}
\textbf{40} (2024), 075008.

\bibitem{LSG11}
S.~G. Lingala, Y.~Hu, E.~DiBella, and M.~Jacob, Accelerated dynamic {MRI}
exploiting sparsity and low-rank structure: k-t {SLR}, \textit{IEEE Trans.
Med. Imaging} \textbf{30} (2011), 1042--1054.

\bibitem{LJZ16}
J.~Liu, X.~Zhang, B.~Dong, Z.~Shen, and L.~Gu, A wavelet frame method with shape prior for ultrasound video segmentation, \textit{SIAM J. Imaging Sci.} \textbf{9} (2016), no.~2, 495--519.

\bibitem{LMD08}
M.~Lustig, D.~L. Donoho, J.~M. Santos, and J.~M. Pauly, Compressed sensing
{MRI}, \textit{IEEE Signal Process. Mag.} \textbf{25} (2008), 72--82.

\bibitem{MJB09}
J.~Mairal, F.~Bach, J.~Ponce, G.~Sapiro, and A.~Zisserman, Non-local sparse
models for image restoration, \textit{2009 IEEE 12th International Conference
on Computer Vision (ICCV)}, 2009, 2272--2279.

\bibitem{MDF01}
D.~Martin, C.~Fowlkes, D.~Tal, and J.~Malik, A database of human segmented
natural images and its application to evaluating segmentation algorithms and
measuring ecological statistics, \textit{Eighth IEEE International Conference
on Computer Vision, Vol II, Proceedings}, 2001, 416--423.

\bibitem{MVL21}
V.~Monga, Y.~Li, and Y.~C. Eldar, Algorithm unrolling: Interpretable, efficient
deep learning for signal and image processing, \textit{IEEE Signal Process.
Mag.} \textbf{38} (2021), 18--44.

\bibitem{MAP15}
A.~Mousavi, A.~B. Patel, and R.~G. Baraniuk, A deep learning approach to
structured signal recovery, \textit{2015 53rd Annual Allerton Conference on
Communication, Control, and Computing}, 2015, 1336--1343.

\bibitem{NDT09}
D.~Needell and J.~A. Tropp, Co{S}a{MP}: iterative signal recovery from
incomplete and inaccurate samples, \textit{Appl. Comput. Harmon. Anal.}
\textbf{26} (2009), 301--321.

\bibitem{NJW06}
J.~Nocedal and S.~J. Wright, \textit{Numerical optimization}, Springer Series
in Operations Research and Financial Engineering, 2nd edn., xxii+664 ,
Springer, New York, 2006.

\bibitem{QXG12}
X.~Qu, D.~Guo, B.~Ning, Y.~Hou, Y.~Lin, S.~Cai, and Z.~Chen, Undersampled {MRI}
reconstruction with patch-based directional wavelets, \textit{Magn. Reson.
Imaging} \textbf{30} (2012), 964--977.

\bibitem{QXH14}
X.~Qu, Y.~Hou, F.~Lam, D.~Guo, J.~Zhong, and Z.~Chen, Magnetic resonance image
reconstruction from undersampled measurements using a patch-based nonlocal
operator, \textit{Med. Image Anal.} \textbf{18} (2014), 843--856.

\bibitem{RSB10}
S.~Ravishankar and Y.~Bresler, {MR} image reconstruction from highly
undersampled k-space data by dictionary learning, \textit{IEEE transactions
on medical imaging} \textbf{30} (2010), 1028--1041.

\bibitem{RSB15}
S.~Ravishankar and Y.~Bresler, Efficient blind compressed sensing using
sparsifying transforms with convergence guarantees and application to
magnetic resonance imaging, \textit{SIAM J. Imaging Sci.} \textbf{8} (2015),
2519--2557.

\bibitem{ROF15}
O.~Ronneberger, P.~Fischer, and T.~Brox, {U-Net}: Convolutional networks for
biomedical image segmentation, N.~Navab, J.~Hornegger, W.~M. Wells, and A.~F.
Frangi (eds.), \textit{Medical Image Computing and Computer-assisted
Intervention, PT III}, vol. 9351, 2015, 234--241.

\bibitem{RYL22}
Y.~Ru, F.~Li, F.~Fang, and G.~Zhang, Patch-based weighted {SCAD} prior for
compressive sensing, \textit{Inform Sciences} \textbf{592} (2022), 137--155.

\bibitem{SBN11}
B.~Stephen, P.~Neal, C.~Eric, P.~Borja, and E.~Jonathan, Distributed
optimization and statistical learning via the alternating direction method of
multipliers, \textit{Found. Trends Inf. Retr.} \textbf{3} (2011), 1--122.

\bibitem{UMP11}
M.~Usman, C.~Prieto, T.~Schaeffter, and P.~G. Batchelor, k-t group sparse: A
method for accelerating dynamic {MRI}, \textit{Magn. Reson. Med.} \textbf{66}
(2011), 1163--1176.

\bibitem{WBR17}
B.~Wen, S.~Ravishankar, and Y.~Bresler, F{RIST}---flipping and rotation
invariant sparsifying transform learning and applications, \textit{Inverse
Problems} \textbf{33} (2017), 074007, 27.

\bibitem{WSJ15}
S.~J. Wright, Coordinate descent algorithms, \textit{Math. Program.}
\textbf{151} (2015), 3--34.

\bibitem{XYY16}
Y.~Xu and W.~Yin, A fast patch-dictionary method for whole image recovery,
\textit{Inverse Probl. Imaging} \textbf{10} (2016), 563--583.

\bibitem{YYS20}
Y.~Yang, J.~Sun, H.~Li, and Z.~Xu, {ADMM-CSNet}: A deep learning approach for
image compressive sensing, \textit{IEEE Trans. Pattern Anal. Mach. Intell.}
\textbf{42} (2020), 521--538.

\bibitem{YYW23}
Y.~Yang, Y.~Wang, J.~Wang, J.~Sun, and Z.~Xu, An unrolled implicit
regularization network for joint image and sensitivity estimation in parallel
{MR} imaging with convergence guarantee, \textit{SIAM J. Imaging Sci.}
\textbf{16} (2023), 1791--1824.

\bibitem{ZZW23}
Z.~Zha, B.~Wen, X.~Yuan, S.~Ravishankar, J.~Zhou, and C.~Zhu, Learning nonlocal
sparse and low-rank models for image compressive sensing: Nonlocal sparse and
low-rank modeling, \textit{IEEE Signal Process. Mag.} \textbf{40} (2023),
32--44.

\bibitem{ZZC16}
Z.~Zhan, J.-F. Cai, D.~Guo, Y.~Liu, Z.~Chen, and X.~Qu, Fast multiclass
dictionaries learning with geometrical directions in {MRI} reconstruction,
\textit{IEEE Trans. Biomed. Eng.} \textbf{63} (2016), 1850--1861.

\bibitem{ZJG18}
J.~Zhang and B.~Ghanem, {ISTA-Net}: Interpretable optimization-inspired deep
network for image compressive sensing, \textit{2018 IEEE/CVF Conference on
Computer Vision and Pattern Recognition (CVPR)}, 2018, 1828--1837.

\bibitem{ZJZ14}
J.~Zhang, D.~Zhao, and W.~Gao, Group-based sparse representation for image
restoration, \textit{IEEE Trans. Image Process.} \textbf{23} (2014),
3336--3351.

\bibitem{ZJC23}
J.~Zhang, B.~Chen, R.~Xiong, and Y.~Zhang, Physics-inspired compressive
sensing: Beyond deep unrolling, \textit{IEEE Signal Process. Mag.}
\textbf{40} (2023), 58--72.

\bibitem{ZKZ17}
K.~Zhang, W.~Zuo, Y.~Chen, D.~Meng, and L.~Zhang, Beyond a {G}aussian denoiser:
Residual learning of deep {CNN} for image denoising, \textit{IEEE Trans.
Image Process.} \textbf{26} (2017), 3142--3155.

\bibitem{ZKZW17}
K.~Zhang, W.~Zuo, S.~Gu, and L.~Zhang, Learning deep {CNN} denoiser prior for
image restoration, \textit{30th IEEE Conference on Computer Vision and
Pattern Recognition (CVPR 2017)}, 2017, 2808--2817.

\bibitem{ZLZ26}
L.~Zhao, Y.~Zhang, X.~Wang, J.~Zhang, H.~Bai, and A.~Wang, A survey on image
compressive sensing: From classical theory to the latest explicable deep
learning, \textit{Pattern Recognit.} \textbf{170} (2026).

\end{thebibliography}
\end{document}